\documentclass{article}
\usepackage{graphicx}
\usepackage{subcaption}
\usepackage{booktabs} 
\usepackage{multirow} 

\usepackage{amsmath}
\usepackage{amsfonts}
\usepackage{amssymb}
\usepackage{amsthm}

\usepackage{amsthm}
\newtheorem{theorem}{Theorem}
\newtheorem{lemma}{Lemma}
\newtheorem{corollary}{Corollary}
\theoremstyle{definition}
\newtheorem{definition}{Definition}

\theoremstyle{remark}
\newtheorem{remark}{Remark}

\usepackage{pdflscape}
\usepackage{pdfpages}
\usepackage{longtable}

\usepackage{algorithm}
\usepackage{algpseudocode}

\newcommand{\N}{\mathbb{N}}
\newcommand{\R}{\mathbb{R}}

\renewcommand{\d}{\textnormal{d}}
\newcommand{\I}{\mathbb{I}}

\renewcommand{\d}{\textnormal{d}}
\renewcommand{\O}{\mathcal{O}}
\renewcommand{\H}{\mathcal{H}}

\newcommand{\X}{\mathcal{X}}
\newcommand{\T}{\mathcal{T}}

\usepackage{todonotes}
\usepackage{apxproof}

\begin{document}
%
\title{Precise Change Point Detection using \\ Spectral Drift Detection\thanks{Funding in the frame of the BMBF project TiM, 05M20PBA, and from the VW-Foundation for the project \textit{IMPACT} is gratefully acknowledged.}}
\author{Fabian Hinder$^1$, Andr\'e Artelt$^1$\footnote{Affiliation with the University of Cyprus}, Valerie Vaquet$^1$ and Barbara Hammer$^1$
%
%
%
\vspace{.3cm}\\
%
Bielefeld University - Cognitive Interaction Technology (CITEC)  \\
Inspiration 1, 33619 Bielefeld - Germany
%
}
\maketitle              
\begin{abstract}
The notion of concept drift refers to
the phenomenon that the data generating distribution
changes over time;
as a consequence machine learning models may become inaccurate and need adjustment.
In this paper we consider the problem of detecting those change points in unsupervised learning.
Many unsupervised approaches 
rely on 
the discrepancy between the sample distributions of two time windows. This procedure is noisy for small windows, hence prone to induce false positives and not able to deal with more than one drift event in a window.
In this paper we rely on structural properties of drift induced signals, which use spectral properties of
kernel embedding of distributions.
Based thereon we derive a new unsupervised drift detection algorithm, investigate its mathematical properties,
and demonstrate its usefulness in several experiments.

{\textbf{Keywords:} Concept Drift $\;\cdot\;$ Drift Detection $\;\cdot\;$ Data Streams $\;\cdot\;$ Non-Parametric Methods $\;\cdot\;$ Kernel Methods $\;\cdot\;$ Spectral Clustering.}
\end{abstract}

\section{Introduction}
One  fundamental assumption in classical machine learning is the fact that observed data are i.i.d.\ according to some unknown probability $P_X$, i.e.\ the data generating process is stationary. 
Yet, this assumption is often violated 
as soon as machine learning faces
data from the real world such as social media entries or measurements of IoT devices, which are subject to continuous change
\cite{DBLP:journals/adt/BifetG20,DBLP:journals/widm/TabassumPFG18a}.
Here, concept drift, i.e.\ changes of the underlying distribution $P_X$, can be caused by seasonal changes, changed demands, ageing of sensors, etc.

The presence of drift can be 
modeled in different ways.
Covariate shift  refers to the  situation of  training and test set having different marginal distributions \cite{5376}. Learning from data streams extends this setting to an unlimited but usually countable stream of observed data 
\cite{asurveyonconceptdriftadaption}.
Here, one distinguishes between virtual and real drift, i.e.\ non-stationarity of the marginal distribution only or also the posterior. Learning technologies for such situations often rely on windowing techniques, and adapt the model based on the characteristics of the data in an observed time window. Active methods explicitly detect drift, usually referring to drift of the classification error, and trigger model adaptation this way, while passive methods continuously adjust the model \cite{DBLP:journals/cim/DitzlerRAP15}. 

In recent years, quite a few approaches on how to deal with concept drift were proposed \cite{DBLP:journals/cim/DitzlerRAP15,Lu_2018}.
They range from non-parametric methods over gradient techniques to ensemble technologies for dealing with streaming data \cite{DBLP:journals/inffus/KrawczykMGSW17}.
The precise pinpointing of the time point of a drift event, commonly referred to as drift detection, constitutes an important sub-problem which 
is used in active learning methods as well as approaches to analyse and understand potentially drifting data streams.
A large number of drift detection methods exist, which partially also  characterize the overarching type of drift
\cite{Aminikhanghahi:2017:SMT:3086013.3086037,DBLP:journals/kais/GoldenbergW19}. 
Many approaches deal with supervised scenarios, aiming for a small interleaved train-test error, and first approaches
for their explanation
\cite{DBLP:journals/corr/WebbLPG17}. As drift can induce problems in all kinds of information analysis and processing tasks where label information is not  available~\cite{D3}, we  focus on the unsupervised case for drift detection.
%
%
Many unsupervised drift detection schemes rely on a comparison of distributions taken from two consecutive time windows \cite{DBLP:journals/kais/GoldenbergW19}.
These techniques can easily lead to false positives; moreover, they are not capable of dealing with consecutive drift events in close  proximity
\cite{SSCI}. 

In this contribution, we aim for a different unsupervised drift detection technology which relies on spectral properties of diverging probability distributions rather than their direct comparison based on a distance measure. 
In targeting this aim, our contribution is twofold:
(1) We provide mathematical insights which enable us to link the presence of drift in a data stream to the spectral structure of the induced kernel matrix. (2) Based on this insight, we propose a new robust drift detection algorithm which is capable of reliably detecting drift events in a given data stream. The approach is based on a spectral analysis of kernel matrices and links the problem of drift detection to spectral clustering. As a result we obtain a very precise estimation of the time point of multiple drift events as we will demonstrate in a number of experiments.

The remainder of this work is structured as follows: first (Section~\ref{sec:setup}), we recall the relevant work from literature and define the problem setup. In Section~\ref{sec:theo} we derive the underlying theory that establishes the regularity properties of the drift induced signals and show its connection to spectral theory (Theorem~\ref{thm:main}). Thereon we derive an efficient drift detection algorithm (Section~\ref{sec:SDD}). In the last section (Section~\ref{sec:exp}) we empirically evaluate the usefulness of our approach for the problem of drift detection. 

\section{Problem Setup and Related Work}
\label{sec:setup}

In the usual, time invariant setup of machine learning, one considers a generative process $P_X$, i.e. a probability measure, on $\R^d$. In this context, one views the realizations of $P_X$-distributed random variables $X_1,...,X_n$ as samples.
Depending on the objective, learning algorithms try to infer the data distribution based on these samples, or, in the supervised setting, the posterior distribution. We will not distinguish between these settings and only consider distributions in general,
subsuming supervised and unsupervised modeling \cite{DBLP:journals/corr/WebbLPG17}.

Many processes in real-world applications are time dependent, so it is reasonable to incorporate time into our considerations. One prominent way to do so, is to consider an index set $\T$, representing time, and a collection of probability measures $p_t$ on $\R^d$, indexed over $\T$, which may change over time \cite{asurveyonconceptdriftadaption}. We will usually assume $\T \subset \R$.
Drift refers to the fact that $p_t$ varies for different time points, i.e.\ 
\begin{align*}
    \exists t_0,t_1 \in \T : p_{t_0} \neq p_{t_1}.
\end{align*}
In this context, we consider a sequence of samples $(X_1,T_1),(X_2,T_2),...$, with $X_i \sim p_{T_i}$ and $T_i \leq T_{i+1}$, as a stream.
Notice, that we will usually use the shorthand notion drift instead of concept drift.
More formally we define:
\begin{definition}
Let $\T \subset \R$ be an open (time) interval, and let $\X$ be the (measurable) data space. 
A \emph{drift process} on $\X$ over $\T$ is a map that assigns a distribution $p_t$ on the data space for every time point $t \in \T$ such that 
$t \mapsto p_t(A)$ is a measurable map for all measurable $A \subset \X$, i.e. a Markov kernel from $\T$ to $\X$. We say that $p_t$ has \emph{drift} if 
$
    \{ (s,t) \in \T^2 \:|\: p_t \neq p_s \}
$
is not a Lebesgue null-set. 

We say that $p_t$ has \emph{abrupt drift only} if there exist time points $t_1 < \dots < t_n, \; t_i \in \T$, called \emph{change points}, and probability measures $P_0,P_1,\dots,P_n$ on $\X$, called \emph{concepts}, such that $P_i \neq P_{i+1}$ and
\begin{align*}
    p_t = \begin{cases}
       P_0 & t < t_1 \\
       P_1 & t \in [t_1,t_2) \\
       \:\:\vdots \\
       P_n & t \geq t_n \\
    \end{cases},
\end{align*}
up to a Lebesgue null-set.
\end{definition}

A relevant problem is to detect drift, i.e.\ identify all time points $t$ such that $p_t$ and $p_{t + \Delta}$ differ, for some small time interval $\Delta > 0$. 
If the underlying distributions are known, this can be done using a metric, i.e. 
\begin{align*}
t \mapsto d(p_t,p_{t+\Delta}).
\end{align*}
We are interested in characterizing the shape of this function in empirical observations. As it will turn out, this will provide important information regarding the drift of the underlying system, which enables a robust localization of the change point in time from observed data. Furthermore, if we take the derivative at $\Delta = 0$ this is closely related to the rate of change or drift speed of the stream. As we will show, this is one of the most relevant quantities in drift detection, as many drift detection schemes rely on it in one way or another.

As $p_t$ are usually not known, most drift detection algorithms which have been proposed in the literature are based on the
comparison of sample characteristics of two time windows \cite{DBLP:journals/kais/GoldenbergW19}. The samples within those time windows are then used to estimate the momentary discrepancies between the current underlying distribution of the stream and a reference time frame.
Popular drift detection methods differ in the choice of distance measures, their transfer into a statistics,  and the allocation of samples to the windows \cite{Lu_2018}. 
Here, we will mainly consider the case of two consecutive windows, that are both sliding along the stream. In the literature, a  common alternative is to use only one sliding window and compare this to a  reference sample; this setup enables the detection of slow drift as well, but it requires a strategy how to chose the  reference window: Depending on the specific  implementation, this may lead to a refractory state after having detected a drift,
where reference samples need to be collected  before being capable of detecting the next drift. Notice that the subsequent statements also take more general windowing techniques into account. 
However, fixed and sliding windows behave in a similar way
for the case of abrupt drift, in which we are interested.

The authors of \cite{DBLP:journals/corr/WebbLPG17} provided a quantification scheme for drift -- drift magnitude --  defined as 
\begin{align*}
    \sigma_{d,l,p_\cdot}(t) = d(p_{[t-2l,t-l]}, p_{[t-l,t]}) 
\end{align*}
where $p_{[a,b]}$ denotes the mean distribution in the time interval $[a,b]$ and $d$ is a distance measure. 
Notice, that this is the specific instantiation of the original definition by \cite{DBLP:journals/corr/WebbLPG17},  which focuses on consecutive time windows. Assuming the samples arrive with equidistant delay, we may estimate $\sigma$ based on the last $2n_l$ samples, where $n_l$ is the number of samples in a time period of length $l$. 
In \cite{DBLP:journals/corr/WebbLPG17} $d$ is chosen as the total variation norm, in \cite{HDDM} Hellinger distance is used, \cite{fail} and \cite{SSCI} use the Maximum Mean Discrepancy~(MMD)~\cite{MMD}, \cite{kdq} use Kullback-Leibler divergence, \cite{KSWIN} use the (feature wise) maximum-norm of the cdfs, and \cite{D3} use classification accuracy of a linear model, which is thus related to the total variance and MMD (with linear kernel).

In \cite{SSCI} the theoretical properties of $\sigma$ are analyzed. The main results are that  $\sigma$ is indeed able to identify drift for a valid distance function by taking on values larger than zero, which also points out the point in time. However, those effects are usually spread out in time. 
Although the work \cite{SSCI} identifies  a characteristic shape of this function in case of isolated, abrupt drift events, a theoretical description of the shape in general is still an open question. 

In the next section we  investigate this problem for the special case of kernel induced metric MMD. We derive a formula to describe the shape of $\sigma$ for general drift processes and show its connection to the problem of measuring the rate of change and, in the case of abrupt drift events, spectral clustering. 

\section{The Shape of the Drift Magnitude}
\label{sec:theo}

The drift magnitude has been introduced to examine the rate of change of a single or multiple features. We will focus on the case where all features are considered at once, resulting in a single time series per data stream. We will focus on Maximum Mean Discrepancy~\cite{MMD} as distance measure, which is defined as 
\begin{align*}
    \textnormal{MMD}(P,Q) = \Vert \mu_P - \mu_Q \Vert_\H,
\end{align*}
where $\mu_P$ and $\mu_Q$ are  kernel embeddings of $P$ and $Q$, respectively, and $\Vert \cdot \Vert_\H$ is the norm in the associated Hilbert-space.

Recall that a kernel $k$ over a space $\X$ is a map \mbox{$k : \X \times \X \to \R$} 
that factors through the inner product of a Hilbert  $\H$ of functions from $\X$ to $\R$, such that $k(x,\cdot) \in \H$ and $\langle k(x,\cdot), f \rangle_\H = f(x)$ for all $f \in \H$. In particular, $k(x,y) = \langle k(x,\cdot), k(y,\cdot) \rangle_\H$; here $k(x,\cdot)$ is usually denoted by $\varphi(x)$ and called the feature map. Note, that $\H$ is uniquely determined by this property. Furthermore, given a measure $P$ on $\X$ we can embed it into $\H$ by integrating $\varphi$ with respect to $P$, i.e. $\mu_P := \varphi(P) := \int \varphi(x) \d P(x)$. If $k$ is universal, i.e. every continuous function can be approximated by functions in the subspace of $\H$ spanned by $\varphi$, then $\Vert P - Q \Vert_k = \Vert \mu_P - \mu_Q \Vert_\H$ defines a norm on the set of all finite measures, which is known as Maximum Mean Discrepancy (MMD).
 
Since we are mainly interested in the embedding of $p_t$ for different values of $t$, we define a map $K : \T \times \T \to \R$ that resembles a kernelized version of the auto-correlation of $p_t$:
\begin{align*}
    K_{t,s} := \langle \mu_{p_t}, \mu_{p_s} \rangle.
\end{align*}
Notice, that integrals commute with the inner product so that it makes sense to consider the mean kernel auto-correlation $K_{w,w'}$ over time-windows $w$ and $w'$. 
Furthermore, as MMD is the norm of a Hilbert space, we can express it for different time points using the kernel auto-correlation:
$
    \Vert p_t - p_s \Vert_k^2 = K_{t,t} - 2K_{t,s} + K_{s,s},
$ 
which again holds for time windows, too.
There is a close connection between the auto-correlation and the kernel matrix: 
\begin{corollary}
Let $X_1,\cdots,X_n \sim p_w$ i.i.d. and $Y_1,\cdots,Y_m \sim p_{w'}$ i.i.d. be samples drawn from the stream during the time windows $w$ and $w'$, with $X_i$ and $Y_j$ independent unless $X_i=Y_j$ 
then it holds
\begin{align*}
    \frac{1}{nm}\sum_i^n \sum_j^m k(X_i,Y_j) \xrightarrow{n,m \to \infty} K_{w,w'},
\end{align*}
in probability with rate $\O(\min\{m,n\}^{-1/2})$, assuming $k$ is bounded.
\end{corollary}

The notion of the kernel auto-correlation will allow us to derive a general formula for the shape of the drift magnitude $\sigma$. However, before we do so, notice that instead of using two sliding windows we can actually use any arbitrary weighting scheme:

\begin{definition}
Let $\T' \subset \T$ be a measurable set. A \emph{weighting scheme} is a finite, signed stochastic kernel $w$ from $\T'$ to $\T$, i.e. $w_t$ is a finite signed measure on $\T$ for all $t \in \T'$, and the map $t \mapsto w_t(A)$ is measurable for all measurable sets $A \subset \T$.
\end{definition}

Notice, that general weighting schemes subsume combinations of fixed, sliding, and growing reference windows with sliding window. 
The special case of two sliding windows corresponds to the weighting scheme induced by the convolution with some function $w_0$, i.e.
$
    w_t(A) \mapsto \int_A w_0(s-t) \d s.
$ 
By setting $w_0 = \I_{[-l,0]} - \I_{[-2l,-l]}$ we obtain the sliding window of length $l$, i.e.
\begin{align*}
    \left\Vert  w_0 * p_t \right\Vert_k^2 := \left\Vert \int p_s w_0(s-t) \d s \right\Vert_k^2 =  \sigma_{\Vert \cdot \Vert_k, l, p_\cdot}(t).
\end{align*}
As it turns out a convolution weighting scheme can be particularly well analyzed as 
 we will see later on in Theorem~\ref{thm:main}.
Before we turn to this special case we consider the general notion of a weighting scheme, as it allows us to derive the following lemma which is applicable to a large variety of drift detection schemes:
(Due to space restrictions all proofs in  this article are contained in the supplemental material.)

\begin{lemma}
\label{lem:commute}
Let $p_t$ be a drift process, $w$ be a weighing scheme, and $k : \X \times \X \to \R$ be a bounded kernel. Assume $\X \subset \R^d$ is compact. Then the eigenfuctions $c_i : \X \to \R$ and eigenvalues $\lambda_i \in \R_{\geq 0}, \; i = 1,\cdots$ of $K_{t,s}$ always exist. Denote by $f_i(t) = \int c_i(x) \d p_t(x)$. Then it holds
\begin{align*}
    \sum_{i = 1}^\infty \lambda_i \left(\int f_i(s) \d w_t(s) \right)^2 &= \left\Vert \int p_s \d w_t(s) \right\Vert_k^2 & \text{in $L^2$.}
\end{align*}
\end{lemma}
\begin{toappendix}
\begin{proof}[Lemma~\ref{lem:commute}]
As $k$ is a kernel on a compact space we may apply Mercers theorem to obtain a representation by means of the eigenfunctions of $k$:
\begin{align*}
    k(x,y) = \sum_{i = 1}^\infty \lambda_i c_i(x) c_i(y),
\end{align*}
with $\lambda_i$ absolutely summable. 
The lemma is now a simple computation:
\begin{align*}
    \left\Vert \int p_s \d w_t(s) \right\Vert_k^2 
    &= \left\langle \iint \varphi(x) \d p_s(x) \d w_t(s), \iint \varphi(y) \d p_{s'}(y) \d w_t(s') \right\rangle 
    \\&= \iiiint \underbrace{\left\langle \varphi(x), \varphi(y) \right\rangle}_{k(x,y)} \d p_s(x) \d w_t(s) \d p_{s'}(y) \d w_t(s')
    \\&= \iiiint \sum_{i = 1}^\infty \lambda_i  c_i(x) c_i(y) \d p_s(x) \d w_t(s) \d p_{s'}(y) \d w_t(s')
    \\&\overset{!}{=} \sum_{i = 1}^\infty \lambda_i \iiiint c_i(x) c_i(y) \d p_s(x) \d w_t(s) \d p_{s'}(y) \d w_t(s')
    \\&= \sum_{i = 1}^\infty \lambda_i \int\!\!\!\underbrace{\int c_i(x) \d p_s(x)}_{= f_i(s)} \d w_t(s) \cdot \int\!\!\!\underbrace{\int c_i(y) \d p_{s'}(y)}_{=f_i(s')} \d w_t(s')
    \\&= \sum_{i = 1}^\infty \lambda_i \left(\int f_i(s) \d w_t(s) \right)^2,
\end{align*}
whereby $!$ holds since $\lambda_i$ is absolute summable and $k$ and thus $c_i$ are bounded.
\end{proof}
\end{toappendix}

Thus, many algorithms consider squared, overlaying, weighted functions that are derived from the drift process. Therefore, it is reasonable to assume that  intricate patterns that occur in drift induced signals are a result of the used weighting; indeed, in~\cite{SSCI} it is shown that  abrupt drift events are distributed in time by the drift magnitude. In the following we will make this more explicit by restricting ourselves to the case of a convectional weighting scheme and consider drift 
in the formal framework of (generalized) derivatives.

\subsection{The Shape of the Drift Magnitude for Sliding Windows}

Notice, that Lemma~\ref{lem:commute} shows a tight connection between general drift and drift that is induced by mixing components, which refers to the fact that there exist probability measures $P_1,\dots,P_n$ and functions $f_1,\dots,f_n : \T \to \R_{\geq 0}$ with $\sum_i f_i(t) = 1$ such that 
\begin{align*}
    p_t = \sum_{i = 1}^n f_i(t) P_i.
\end{align*}
In the kernel space, the ``components'' $P_i$ are provided by the eigenfunctions $c_i$ and, since eigenfunctions are orthonormal, the roles of the $f_i$ coincide. Mixing components are interesting as they are easy to analyze: If we denote by $w_l$ and $h_l$ the windowing and shape function \cite[Theorem 1]{SSCI}, respectively, then it holds
\begin{align*}
    \sigma_{\Vert \cdot \Vert_k,l,p_\cdot}(t) = \left\Vert \sum_{i = 1}^n w_l*f_i(t) P_i \right\Vert_k^2 = \sum_{i,j}^n (h_l * f_i')(t) \cdot (h_l * f_j')(t) \langle \mu_{P_i}, \mu_{P_j} \rangle_\H.
\end{align*}
As we will see, this idea carries over to general drift processes, assuming we make use of MMD as metric. Indeed, it is the most fundamental quantity in the context of drift detection using sliding windows.
To make this precise we need to generalize the notion of differentiability to drift processes. Intuitively, a drift process is differentiable if any derived quantity depends differentiable on time. Formally we define:

\begin{definition}
Let $p_t$ be a drift process.
We say that $p_t$ is \emph{(distributional/weak) differentiable} if $t \mapsto p_t(A)$ is (distributional/weak)\footnote{We recall the relevant definitions in the supplement.}
differentiable for all measurable sets $A \subset \X$.
\end{definition}
\begin{toappendix}
\subsection{Distributions}
\label{sec:dist}
\begin{definition}
Let $\Omega \subset \R^d$ an open, not empty subset. Denote by $\mathcal{D}(\Omega)$ the space of all test functions, i.e. smooth functions with compact support, equipped with the topology such that $\phi_n \to \phi$ in $\mathcal{D}$ if and only if there exists a compact set $K$ such that $\textnormal{supp } \phi_i \subset K,\; \textnormal{supp } \phi \subset K$ and $\partial^\alpha \phi_i \to \partial^\alpha \phi$ uniformly for all $\alpha$. The \emph{space of distributions} on $\Omega$ is then given as dual space of $\mathcal{D}(\Omega)$; we refer to the elements of this space as \emph{distributions}. 

There is a canonical embedding from the space of locally integrable functions $L^1_{\textnormal{loc}}(\R^d)$ into the space of distributions $f \mapsto T_f$ which is induced by $T_f(\phi) = \int f(x) \phi(x) \d x$ where $\phi$ is a test function. We call the distributions that can be represented this way \emph{regular} distributions.
\end{definition}

One important distribution, that is not a regular distribution, is the Dirac-impulse which is induced by $\delta : \phi \mapsto \phi(0)$. 

Distributions are relevant for our problem as they admit a notion of derivative that is not available if we consider functions only.

\begin{definition}
Let $\Omega \subset \R^n$ and $\alpha \in \N_0^n$.
We say that a distribution $T$ on $\Omega$ is \emph{(distributional) $\alpha$-differentiable} iff there exists a distribution $T'$ such that \begin{align*}
    T'(\phi) = (-1)^{|\alpha|} T( D^\alpha \phi)
\end{align*}
for all test functions $\phi$; here $D^\alpha$ denotes the total derivative with respect to $\alpha$ and $|\alpha|$ the degree of $\alpha$. We then refer to $T'$ as the \emph{$\alpha$-th derivative} of $T$ and denote it by $D^\alpha T$ (or $T'$ if $n = 1$).
\end{definition}
Note, that $T'$ is uniquely determined by this property so that the notation is valid.

\begin{definition}
Let $\Omega \subset \R^n$, $\alpha \in \N_0^n$ and $f \in L^1_{\textnormal{loc}}(\Omega)$ be a locally integrable function. We say that $f$ is \emph{weak $\alpha$-differentiable}, iff there exists a locally integrable function $f' \in L^1_{\textnormal{loc}}(\Omega)$ such that $D^\alpha T_f = T_{f'}$. We refer to $f'$ as the \emph{$\alpha$-th weak derivative of $f$} and denote it by $D^\alpha f$ (or $f'$ if $n = 1$). 
\end{definition}
Notice, that $D^\alpha f$ is uniquely determined by the property and that, if $f$ is $\alpha$-differentiable in the classical sense then the classical and weak derivative coincide.

A very important example of a function that is neither differentiable nor weak differentiable but distributional differentiable is the indicator function $\I_{(-\infty,0]}(x)$ which is 1 if $x$ is smaller or equal to 0 and 0 otherwise. Its first derivative is given by the Dirac-impulse, i.e. $\I'_{(-\infty,0]} = -\delta$. 

This function is of particular interest for our purposes, as we may write a drift process with abrupt drift only, with change points $t_1 < \dots < t_n$ and concepts $P_0,P_1,\dots,P_n$ as 
\begin{align*}
    p_t = \sum_{i = 0}^{n} \I_{(t_{i},t_{i+1}]} P_i,
\end{align*}
where $t_0 = -\infty$ and $t_{n+1} = \infty$. It is now obvious that the (distributional) derivative of $p_t$ is given by
\begin{align*}
    p_t' = \sum_{i = 1}^n \delta_{t_i} (P_i - P_{i-1}),
\end{align*}
where $\delta_t : \phi \mapsto \phi(t)$ denotes the shifted Dirac-impulse. 

\begin{definition}
Let $T,T_1,T_2$ be distributions on $\R^d$. The \emph{convolution} of $T$ with a test function $\phi$ is defined as $(T * \phi)(x) = T(\phi(x-\cdot))$. Notice, that this is a test function. 

If at least one of $T_1$ or $T_2$ have bounded support, then the \emph{convolution of distributions} is the distribution $T_1 * T_2$ defined by $(T_1 * T_2) * \phi = T_1 * (T_2 * \phi) $ for all test functions $\phi$.
\end{definition}

Notice, that the commonly known rule for derivatives of convolutions holds for the convolution of a distribution and a test function, i.e.
\begin{align*}
    D^\alpha (T * \phi) = (D^\alpha T) * \phi = T * (D^\alpha \phi),
\end{align*}
assuming $D^\alpha T$ exists. This statement can be extended to distributions:

\begin{lemma}
\label{lem:convolve}
Let $T_1, T_2$ be two $\alpha$-differentiable distributions with bounded support, then it holds
\begin{align*}
    D^\alpha (T_1 * T_2) = (D^\alpha T_1) * T_2 = T_1 * (D^\alpha T_2).
\end{align*}
In particular, let $f, g : \R \to \R$ be functions and let $f$ be distributional differentiable and $G$ be an antiderivative    of $g$. Assume that $g$ and $G$ have bounded support, then it holds
\begin{align*}
    f * g = T_f' * G.
\end{align*}
\end{lemma}
\begin{proof}
Let $\phi$ be any test function. It holds 
\begin{align*}
    D^\alpha ((T_1 * T_2) * \phi) &= D^\alpha (T_1 * (T_2 * \phi)) \\&= T_1 * (D^\alpha (T_2 * \phi)) \\&= T_1 * ((D^\alpha T_2) * \phi) \\&= (T_1 * (D^\alpha T_2)) * \phi.
\end{align*}
For the other equality use that $T_1 * T_2 = T_2 * T_1$. For the second statement use that $G' = g$ and consider $T_f$ and $T_g$.
\end{proof}
\end{toappendix}

\begin{remark} Notice that
\begin{enumerate}
    \item 
for $\X = \R^d$ it holds that $p_t$ is (distributional/weak) differentiable if and only if $t \mapsto \int f(x) \d p_t(x)$ is (distributional/weak) differentiable for all measurable, bounded continuous, or bounded smooth functions $f : \X \to \R$.
In particular, if $p_t$ is (distributional/weak) differentiable then the functions $f_i$ from Lemma~\ref{lem:commute} are (distributional/weak) differentiable.
\item mixing processes are (distributional/weak) differentiable if all $f_i$ are (distributional/weak) differentiable.
\item drift processes with abrupt drift only are distributional differentiable.
\end{enumerate}
\end{remark}

Using this definition, we can now relate the drift magnitude to the notion of derivatives. As it turns out, the signals obtained by the drift magnitude are essentially contorted derivatives:

\begin{theorem}
\label{thm:main}
Let $p_t$ be a (distributional/weak) differentiable drift process, $w : \R \to \R$ be a (weighting) function, and $k:\X \times \X \to \R$ be a bounded kernel, i.e. $k(x,x) < C$ for all $x \in \X$. Denote by $W$ a antiderivative (primitive integral) of $w$. Assume $\X$ is compact, $\T \subset \R$ is a bounded domain, and $w$ and $W$ have bounded support. Then the eigenfunctions of $K_{s,t}$ (considered as an integral operator) have (distributional/weak) derivatives $f_i'$ (with eigenvalues $\lambda_i$) and it holds
\begin{align*}
    \sum_{i = 1}^\infty \lambda_i (W * f'_i)^2(t) &= \Vert w * p_t \Vert_k^2
\end{align*}
in $L^2$,
where $\cdot* \cdot$ denotes the convolution operator. 

In particular, if $p_t$ has abrupt drift only, with change points $t_1,\dots,t_n$ and concepts $P_0,P_1,\dots,P_n$
then the expression simplifies to
\begin{align*}
    \sum_{i = 1}^\infty \left( \sum_{j = 1}^n s_{i,j} W(t-t_j)\right)^2
    &= \Vert w * p_t \Vert_k^2,
\end{align*}
with $\sum_{i = 1}^\infty s_{i,j}^2 = \Vert P_j-P_{j-1} \Vert_k^2$. 
Furthermore, if the length of the support of $W$ is smaller than the minimal distances between the $t_i$, i.e. if $\textnormal{supp } W \subset [a,a+l]$ and \mbox{$l < \min \{ t_{i+1} - t_i \:|\: i=1,...,n-1 \}$}, then only one summand is active at once and the sum commutes with square roots:
\begin{align*}
    \sum_{j = 1}^n \Vert P_j-P_{j-1} \Vert_k W(t-t_j) &= \Vert w * p_t \Vert_k.
\end{align*}
\end{theorem}
\begin{toappendix}
\begin{proof}[Theorem~\ref{thm:main}]
Using Lemma~\ref{lem:commute} we can represent 
\begin{align*}
    K_{s,t} = \sum_{i = 1}^\infty \tilde{\lambda_i} \tilde{f_i}(s) \tilde{f_i}(t),
\end{align*}
with $\tilde{\lambda}_i$ absolute summable and $\tilde{f}_i$ distributional differentiable, by assumption.

On the other hand, $K_{s,t}$ is a kernel and thus, by Mercers theorem, admits eigenvectors $\lambda > 0$ and eigenfunctions $f_i : \T \to \R$ such that $\sum_i \lambda_i < \infty$ and
\begin{align*}
    K_{s,t} = \sum_{i = 1}^\infty \lambda_i f_i(t) f_i(s).
\end{align*}
In particular, since $f_i$ is an eigenfunction of $K_{s,t}$ it holds
\begin{align*}
         \lambda_k f_k(t) 
      &= \int K_{s,t} f_k(s) \d s 
    \\&= \int \sum_{i = 1}^\infty \tilde{\lambda}_i \tilde{f}_i(t) \tilde{f}_i(s) f_k(s) \d s
    \\&\overset{!}{=} \sum_{i = 1}^\infty \tilde{\lambda}_i \int \tilde{f}_i(s) f_k(s) \d s \;\cdot\; \tilde{f}_i(t),
\end{align*}
where $!$ holds, since $\tilde{\lambda_i}$ is absolutely summable. Now, by Hölders inequality it follows that $\int \tilde{f}_i(s) f_k(s) \d s \leq \int \tilde{f}_i(s)^2 \d s \cdot \int f_k(s)^2 \d s \leq C$, thus $\tilde{\lambda}_i \int \tilde{f}_i(s) f_k(s) \d s$ is absolutely summable, too. As a consequence, $f_k$ has to be distributional differentiable as well, since it can be written as the weak limite, of bounded, differentiable functions. The first statement now follows using the same arguments as in the proof of Lemma~\ref{lem:commute} by Lemma~\ref{lem:convolve}.

The second part follows from the fact that $K_{s,t}$ is locally constant with jumps along $t_i$. Thus, $f_k$ is locally constant with jumps at $t_i$ and therefore admits the distributional derivative 
\begin{align*}
    \sqrt{\lambda_k} f_k(t) = \sum_{i = 1}^n s_{i,j} \delta_{t_i}
\end{align*}
and it holds $(W * \delta_t) = W(\cdot-t)$. The sum property follows by choosing $w$ with sufficiently small support, observing that $s_{i,j}$ does not depend on $w$, and using \cite[Corollary 1]{SSCI}.
\end{proof}
\end{toappendix}

Notice, that the last formulation is exactly the statement of \cite[Corollary 1]{SSCI}, observing that the antiderivative of the windowing function $\I_{[-l,0]} - \I_{[-2l,-l]}$ is exactly the shape-function $h_l$ from \cite[Theorem 1]{SSCI}. 
Furthermore, Theorem~\ref{thm:main} establishes a direct connection between the drift magnitude and the rate of change in terms of the derivatives of the eigenfunctions of the kernel auto-correlation. 

Since a large variety of drift detection algorithms can be analyzed using Lemma~\ref{lem:commute}, including 
D3~\cite{D3}, KS-Win~\cite{KSWIN}, MMDDDM~\cite{fail}, and ShapeDDM~\cite{SSCI}, we observe that they all analyze a weighted sum of smoothed out eigenfunctions. The main idea of spectral drift detection is to instead analyze the eigenfunctions directly. In the next section we will derive a new drift detection algorithm by adapting spectral clustering~\cite{SpectralClustering} to the task at hand.
 
\section{Spectral Drift Detection Method}
\label{sec:SDD}

\begin{figure}[t]
    \centering
    \begin{minipage}[b]{0.22\textwidth}
    \centering
    \includegraphics[width=\textwidth]{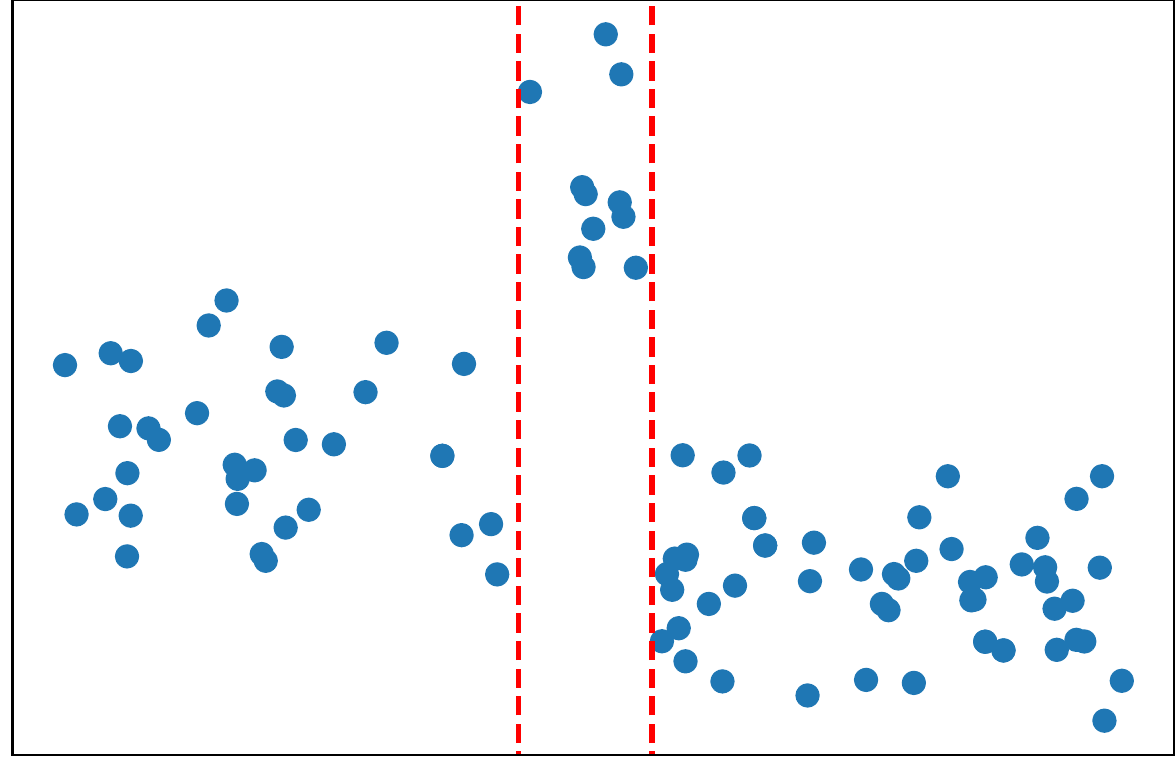}
    \subcaption{Toy dataset with drift}
    \end{minipage}
    \begin{minipage}[b]{0.22\textwidth}
    \centering
    \includegraphics[width=0.8\textwidth]{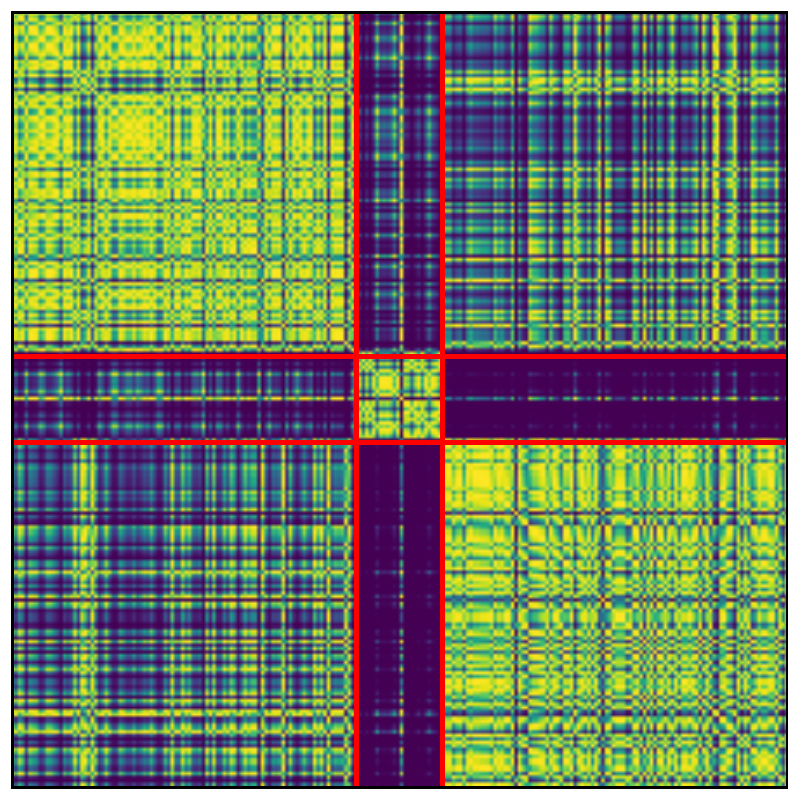}
    \subcaption{Kernel matrix}
    \end{minipage}
    \begin{minipage}[b]{0.22\textwidth}
    \centering
    \includegraphics[width=\textwidth]{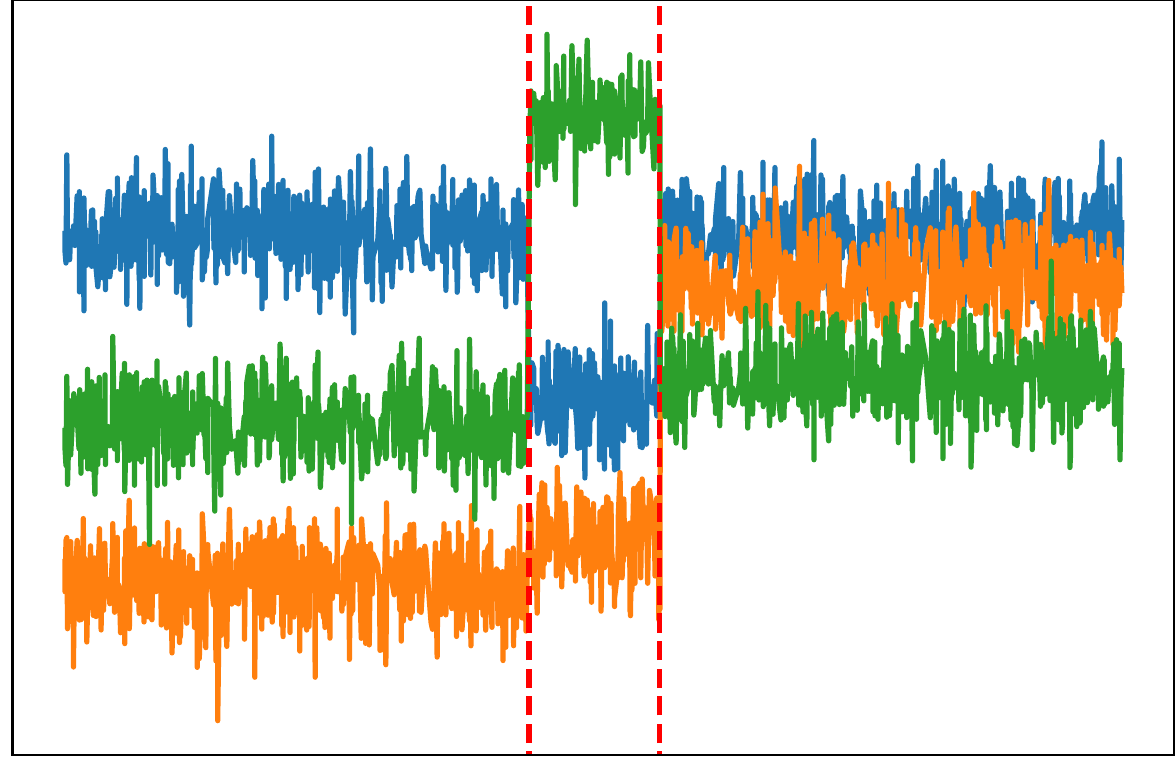}
    \subcaption{First 3 eigenvectors}
    \end{minipage}
    \begin{minipage}[b]{0.22\textwidth}
    \centering
    \includegraphics[width=\textwidth]{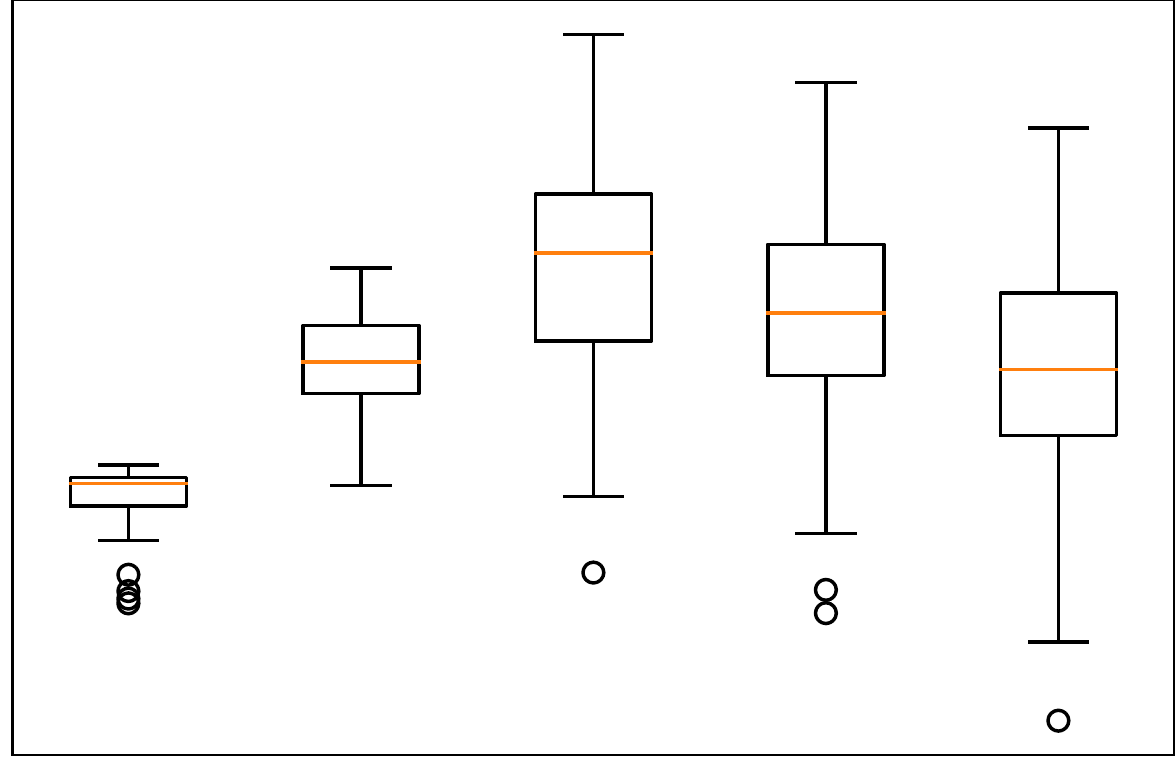}
    \subcaption{Results of cross-validation}
    \label{fig:illustration:cross_validation}
    \end{minipage}
    \caption{Four stages of the Spectral Drift Detection Method. Red lines (a,b,c) mark drift events. (d) shows R2-score for Decision Trees with varying number of leafs (1-5); as can be seen optimal score is obtained at 3 leafs/drifts.}
    \label{fig:illustration}
\end{figure}

The core observation that leads to the \emph{Spectral Drift Detection Method (SDDM)} is the following: If $p_t$ has abrupt drift only, then a kernel matrix obtained from samples drawn from the drift process takes on the shape of a block matrix, where the boundaries of the blocks corresponds to the change points and the blocks correspond to samplings from the mean auto-correlation. We have illustrated this in Fig.~\ref{fig:illustration}. Formally, we obtain the following statement:

\begin{theorem}
\label{thm:SSD}
Let $p_t$ be a drift process with abrupt drift only at $t_1 < \cdots < t_n$, on $\T = [t_0, t_{n+1}]$, and let $k : \X \times \X \to \R$ be a kernel. Then there exists a matrix $\tilde{K} \in \R^{(n+2) \times (n+2)}$ such that
\begin{align*}
    K_{t,s} = \sum_{i,j = 0}^{n+1} \I_{[t_i, t_{i+1})}(t) \tilde{K}_{i,j} \I_{[t_j, t_{j+1})}(s).
\end{align*}
And it holds that the eigenvalues of $K$ and $\tilde{K}$ coincide and if $v_l$ is the $l$-th eigenvector of $\tilde{K}$, then 
\begin{align*}
    f_l(t) = \sum_{i = 0}^{n+1} \I_{[t_i, t_{i+1})}(t) \frac{(v_l)_i}{t_{i+1}-t_i}
\end{align*}
is the $l$-th eigenfunction of $K$. 

\end{theorem}
\begin{toappendix}
\begin{proof}[Theorem~\ref{thm:SSD}]
The first part follows directly from the definitions. The second part follows by the observation that $\int f_k(t) f_l(t) \d t = v_k^\top v_l$.
\end{proof}
\end{toappendix}

As stated above, SDDM directly analyzes the eigenvectors of the kernel matrix. This allows us to connect it to other methods based on MDD, which usually consider
$ 
    \widehat{\textnormal{MMD}} = w^\top \hat{K} w,
$  
where $\hat{K}$ is the kernel matrix and \ $w = (-1/n_0,\dots,-1/n_0,1/n_1,\dots,1/n_1)$ with $n_0$ and $n_1$ the length of the reference and current window, respectively. It is easy to see that sliding along the stream (MMDDDM, ShapeDDM) or using auto-cuts (comparable to ADWIN) corresponds to an optimization problem over those vectors $w$, which essentially approximates, and in case of a single drift event is equivalent to, the computation of the larges eigenfunction of $K_{s,t}$. Furthermore, if we consider $p_t$ as a mixing problem, as suggested by Lemma~\ref{lem:commute}, then the eigenvectors of the kernel matrix are approximations of the mixing functions $f_i(t)$.

Due to the connection between the kernel matrix of streaming data and the kernel auto-correlation function, we obtain the drift detection algorithm by applying a modified version of Spectral Clustering~\cite{SpectralClustering} to the kernel matrix. Intuitively speaking, we form clusters of those time points with the same distribution. 
Recall that the spectral clustering algorithm is defined as follows:
\begin{enumerate}
    \item Compute a similarity matrix $K$ from data points $x_1,\cdots,x_n$
    \item Compute the row sum matrix $D = \textnormal{diag}(K \mathbf{1})$
    \item Compute the (normalized) Laplacian $L = D-K$ ($L = I - D^{-\frac{1}{2}} K D^{-\frac{1}{2}}$ or $L = I - D^{-1}K$)
    \item Compute the first $k$ eigenvectors $v_1,\dots,v_k$ of $L$ (eigenvectors with smallest eigenvalues)
    \item Apply a partition algorithm, e.g. $k$-means,
    to the rows of the eigenvectors, i.e. $\{ (v_{1j},\cdots,v_{kj})^\top \:|\: j = 1,\cdots,n\}$ \label{spectral_clustering:last_step}
\end{enumerate}

A common choice for the partition algorithm is $k$-means. However, as we want to find clusters of consecutive time points we need another partitioning method. To do so recall that $k$-means tries to find clusters $C$ such that the inner cluster variance is minimized. 
In addition, we want that the corresponding time points form time intervals, i.e. we want to find time intervals (correspond to clusters) such that the inner variance of the associated samples is minimized.
Notice, that this is exactly the objective of decision tree regression with variance reduction, i.e. we replace Step~\ref{spectral_clustering:last_step} in the Spectral Clustering algorithm by: 
\begin{itemize}
    \item[\ref{spectral_clustering:last_step}'. ] Train a decision tree to predict $t_j \mapsto (v_{1j},\cdots,v_{kj})$, with $t_j$ the arrival time of $x_j$. 
The boundary points of the leafs correspond to the drift events.
\end{itemize} 

Notice, that the plain CART algorithm is not suited to find optimal trees, as it produces far too many leaf nodes. To solve this issue we suggest to limit the number of leaf nodes to the number of expected drift events. To obtain the number of drift events, observe that the mean of $v_{ij}$ is the same for all $t_j$ between two drift events and different otherwise. Thus, the number of change points is equivalent to the optimal number of leafs, due to the bias-variance tradeoff. Therefore, we can obtain the number of drift events via cross-validation. We illustrated an example of the score for different numbers of leafs in Fig.~\ref{fig:illustration:cross_validation}. This completes the derivation of the Spectral Drift Detection Method, which is summarized in Algorithm~\ref{alg:SDD}.

\begin{algorithm}[!t]
	\caption{Spectral Drift Detection Method}
	\label{alg:SDD}
	\begin{algorithmic}[1]
		\Procedure{SDDM: Spectral Drift Detection Method}{$(x_i)$ data stream, $n_\text{eigen}$ number of eigenvectors, $n_\text{itr}$ number of iterations in cross-validation, $k_{\max}$ maximal number of drifts} \;
		\State Initialize Window $W \gets []$\;
		\While{Not at end of stream $(x_i)$} \;
	    \State $W \gets W + [x_i]$ \Comment{Add new sample}
	    \If{$|W| > n_{\max}$}
	        \State $\textsc{pop}(W)$ \Comment{Drop oldest sample}
	    \EndIf
	    \If{$|W| > n_{\min}$}
	        \State $K \gets \textsc{ComputeKernel}(W)$ \Comment{Compute similarity matrix} \label{alg:SDD:start_offline}
	        \State $D'_{ii} \gets \left(\sum_{j = 1}^{|W|} K_{ij} \right)^{-\frac{1}{2}}$
	        \State $L = I - D' K D'$
	        \State $V \gets \textsc{eigen}(L,n_\text{eigen})$ 
	        \Comment{Compute $n_\textnormal{eigen}$ first eigenvectors}
	        \State $k \gets 
	        \textsc{FindNumberOfLeafsByCrossValidation}(n_\text{itr},k_{\max},V)$
	        \If{$k > 1$}
	            \State $t \gets \textsc{TrainTreeWithLeafs}(V,k)$
	            \State $d_1,\cdots,d_k \gets \textsc{ExtractDecisionThresholds}(t)$
	            \State Alert drift at $d_1,\cdots,d_k$
	        \EndIf\label{alg:SDD:end_offline}
	    \EndIf
		\EndWhile
		\EndProcedure
	\end{algorithmic}
\end{algorithm}

\subsection{Implementation Details and Algorithmic Properties}
We will now discuss some of the aspects of SDDM in more detail.

\begin{figure}[t]
  \centering
  \begin{minipage}[b]{0.2\textwidth}
    \centering
    \includegraphics[width=\textwidth]{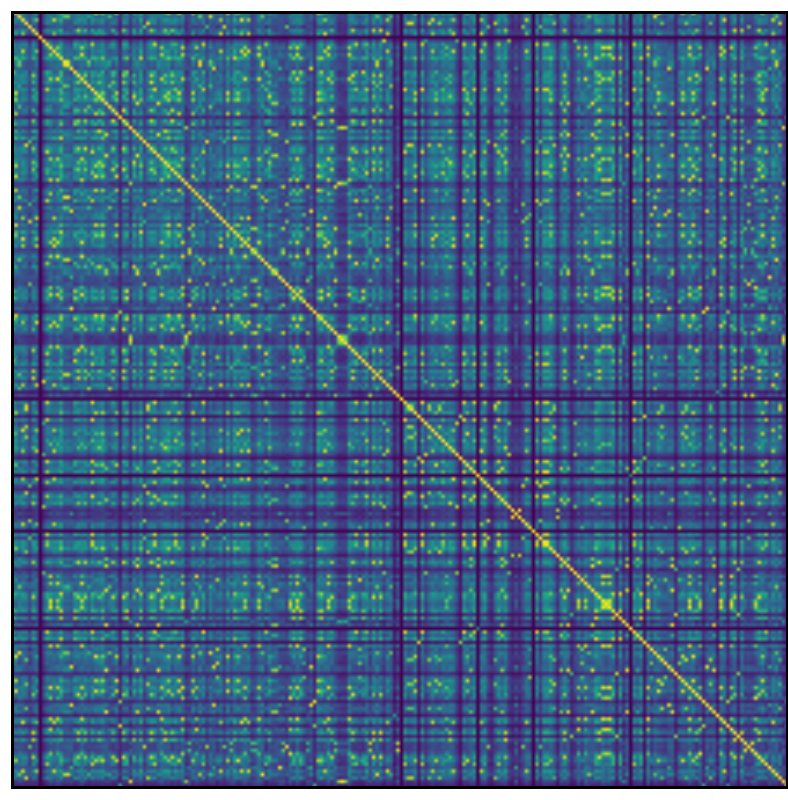}
    \subcaption{Gauss\label{fig:kernels:rbf}}
  \end{minipage}
  \hfill
  \begin{minipage}[b]{0.2\textwidth}
  \centering
    \includegraphics[width=\textwidth]{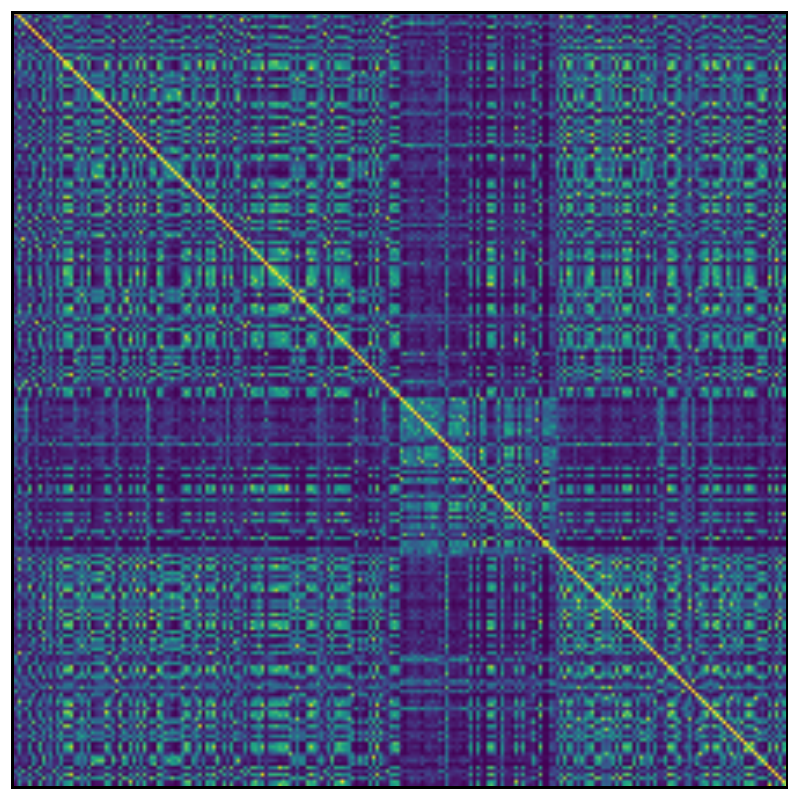}
    \subcaption{MomentTree\label{fig:kernels:tree}}
  \end{minipage}
  \hfill
  \begin{minipage}[b]{0.55\textwidth}
  \centering
    \includegraphics[width=0.95\textwidth]{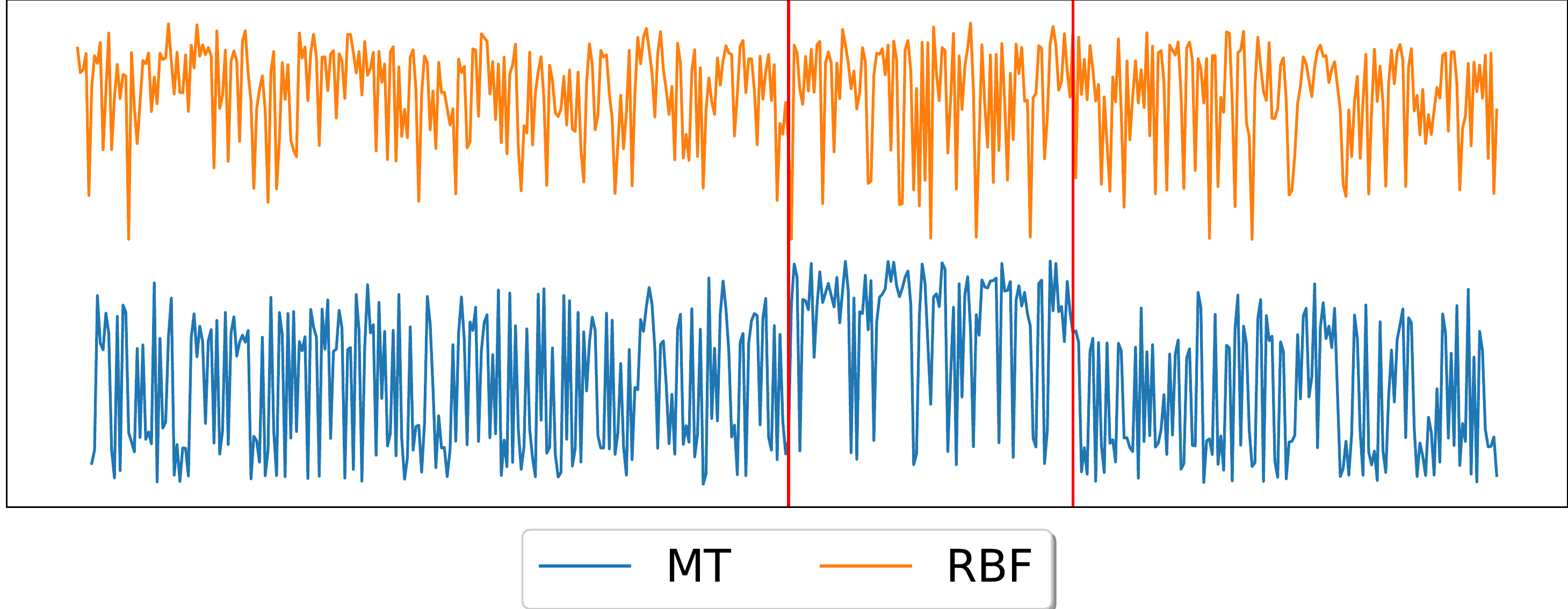}
    \subcaption{First eigenvector\label{fig:kernels:eigen}}
  \end{minipage}
  \caption{Comparison of MomentTree (MT) and Gauss (RBF) kernel matrix on ForestCovertype dataset (ABA, $|B| = 100$; see Section~\ref{sec:exp}) and first eigenvector of both kernels matrices. MT shows significant difference between drifts.}
  \label{fig:kernels}
\end{figure}

\paragraph{MomentTree-Kernel} Although any kernel can be used for SDDM, we suggest to make use of a kernel based on MomentTrees~\cite{MomentTree} which seems to provide a better distance measure than the commonly used Gauss-kernel~\cite{ida}. 

MomentTrees is a method for decision tree based conditional density estimation; in the context of drift they are trained to predict the arrival time $t$ based on the sample $x$~\cite{ida}, i.e. they learn $x \mapsto t$. The idea is derived from the observation that, by using conditional density estimation, the data space can be segmented into regions with the same drift behavior~\cite{esann}. 
Due to this fact, the trees will find a segmentation of the data space that provides a particularly good binning to measure the distance between $p_t$ and $p_s$ \cite{ida}. 

To obtain a kernel we use tree similarities~\cite{breiman2003random}, i.e. we train RandomForest of MomentTrees on the current window, the similarity of two points is given by the ratio of trees for which both point belong to the same leaf region.
We illustrated the benefit of using MomentTree-kernels over Gauss-kernels in Fig.~\ref{fig:kernels}.

\paragraph{Decision Threshold / $p$-Value } In contrast to many other drift detection methods, SDDM is not based on a statistical test to determine whether or not a drift occurred at certain time. Therefore, we do not need a threshold $p$- or $\alpha$-value which is usually hard to estimate. Instead, SDDM uses a decision process, which is based on cross-validation, to determine the number of drift points. Although the cross-validation needs some parameters (number of iterations, train-test-split ratio), the method turns out to be robust in a large spectrum of choices.

\paragraph{Batch and Offline-Mode} Notice, that SDDM detects all drift events within a given window. Therefore, it can be applied in offline mode, i.e. we only run lines~\ref{alg:SDD:start_offline} to~\ref{alg:SDD:end_offline} in Algorithm~\ref{alg:SDD}. This is also beneficial for applying SDDM in an online setup as it involves several, computational expensive tasks as the training of the MomentTrees, the computation of the eigenvectors, and performing the cross-validation. Those issues can be solved by running lines~\ref{alg:SDD:start_offline} to~\ref{alg:SDD:end_offline} in a batch mode, i.e. instead of running them every time a new sample is received, we only execute them after a certain number of new samples are received.

\paragraph{Postprocessing} As SDDM will alert the same drift events as long as they are present in the current window, it is required to filter out events that are found multiple times. 
Though, the selection usually pinpoints the exact same time-stamp during several runs, we postprocess the selection using Ward clustering.
Furthermore, as false alarms usually only occur in a single batch we can reduce the number of false alarms by considering the drift events that are found in most batches (that can potentially contain them) only.

\section{Experiments}
\label{sec:exp}
We performed empirical experiments 
to evaluate capability of our method to detect abrupt drift and determine the time point of the drift.

\begin{figure}[t]
    \centering
    \includegraphics[width=\textwidth]{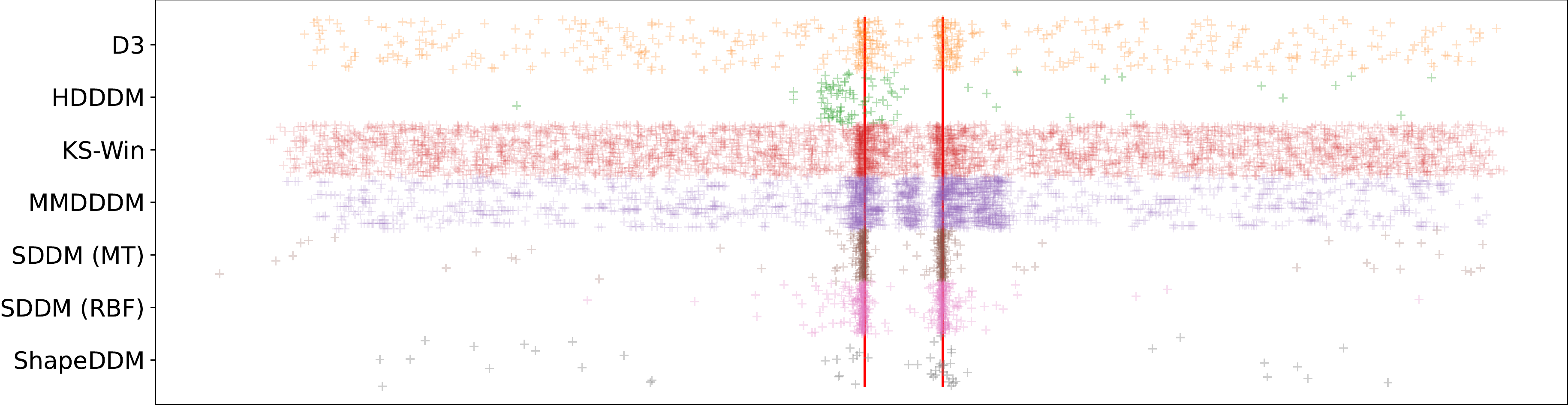}
    \caption{Results on Nebraska Weather dataset (ABC, $|B| = 100$) over 150 runs. $x$-axis represents time, $y$-axis and color represent method (and run), crosses mark found drift events, red lines mark true drift events.}
    \label{fig:results_weather}
\end{figure}

\begin{table}[]
    \centering
    \caption{Results over 150 runs. Table shows mean $\beta$-score ($\beta = 0.5$) and standard deviation. Numbers below stream types are B-length and maximal delay.}
    \small
\begin{tabular}{ll|r@{$\pm$}r@{\quad}r@{$\pm$}r@{\quad}r@{$\pm$}r@{\quad}r@{$\pm$}r@{\quad}r@{$\pm$}r}
\toprule
Dataset & \multicolumn{1}{l|}{Method}&  \multicolumn{2}{c}{AB} & \multicolumn{4}{c}{ABA} &  \multicolumn{4}{c}{ABA} \\
&\multicolumn{1}{r|}{}&
\multicolumn{2}{c}{--/25}&\multicolumn{2}{c}{25/12}&\multicolumn{2}{c}{50/25}&\multicolumn{2}{c}{25/12}&\multicolumn{2}{c}{50/25} \\
\midrule
\multirow{7}{*}{\rotatebox[origin=c]{90}{\parbox{2cm}{Electricity\\Market Prices}}} 
        & SDDM (MT) &  $\mathbf{0.87}$ &  $\mathbf{0.27}$ &  $\mathbf{0.20}$ &  $\mathbf{0.36}$ &  $\mathbf{0.65}$ &  $\mathbf{0.40}$ &  $\mathbf{0.86}$ &  $\mathbf{0.22}$ &  $\mathbf{0.95}$ &  $\mathbf{0.10}$ \\
        & SDDM (RBF) &           $0.50$ &           $0.46$ &           $0.00$ &           $0.00$ &           $0.01$ &           $0.07$ &           $0.64$ &           $0.36$ &           $0.67$ &           $0.34$ \\
& D3 &           $0.17$ &           $0.28$ &           $0.09$ &           $0.16$ &           $0.13$ &           $0.17$ &           $0.29$ &           $0.15$ &           $0.28$ &           $0.16$ \\
        & HDDDM &           $0.39$ &           $0.47$ &           $0.00$ &           $0.00$ &           $0.01$ &           $0.07$ &           $0.18$ &           $0.23$ &           $0.25$ &           $0.24$ \\
        & KSDD &           $0.09$ &           $0.09$ &           $0.07$ &           $0.08$ &           $0.14$ &           $0.10$ &           $0.18$ &           $0.09$ &           $0.21$ &           $0.10$ \\
        & MMDDDM &           $0.14$ &           $0.10$ &           $0.02$ &           $0.06$ &           $0.14$ &           $0.16$ &           $0.17$ &           $0.08$ &           $0.24$ &           $0.09$ \\
        & ShapeDDM &  $\mathbf{0.88}$ &  $\mathbf{0.27}$ &           $0.00$ &           $0.04$ &           $0.03$ &           $0.11$ &           $0.37$ &           $0.21$ &           $0.47$ &           $0.09$ \\
\hline
\multirow{7}{*}{\rotatebox[origin=c]{90}{\parbox{2cm}{Forest\\Covertype}}} 
        & SDDM (MT) &           $0.92$ &           $0.15$ &  $\mathbf{0.69}$ &  $\mathbf{0.32}$ &  $\mathbf{0.89}$ &  $\mathbf{0.13}$ &  $\mathbf{0.66}$ &  $\mathbf{0.24}$ &  $\mathbf{0.87}$ &  $\mathbf{0.19}$ \\
        & SDDM (RBF) &           $0.03$ &           $0.16$ &           $0.00$ &           $0.00$ &           $0.00$ &           $0.00$ &           $0.00$ &           $0.00$ &           $0.01$ &           $0.06$ \\
& D3 &           $0.44$ &           $0.27$ &           $0.21$ &           $0.18$ &           $0.26$ &           $0.15$ &           $0.29$ &           $0.15$ &           $0.31$ &           $0.12$ \\
        & HDDDM &           $0.28$ &           $0.38$ &           $0.03$ &           $0.11$ &           $0.12$ &           $0.21$ &           $0.17$ &           $0.21$ &           $0.33$ &           $0.20$ \\
        & KSDD &           $0.06$ &           $0.02$ &           $0.07$ &           $0.03$ &           $0.09$ &           $0.03$ &           $0.09$ &           $0.04$ &           $0.10$ &           $0.03$ \\
        & MMDDDM &           $0.17$ &           $0.09$ &           $0.11$ &           $0.12$ &           $0.29$ &           $0.15$ &           $0.19$ &           $0.13$ &           $0.21$ &           $0.08$ \\
        & ShapeDDM &  $\mathbf{0.98}$ &  $\mathbf{0.07}$ &           $0.02$ &           $0.10$ &           $0.22$ &           $0.25$ &           $0.47$ &           $0.11$ &           $0.47$ &           $0.11$ \\
\hline
\multirow{7}{*}{\rotatebox[origin=c]{90}{\parbox{2cm}{Random\\RBF}}} 
        & SDDM (MT) &  $\mathbf{0.93}$ &  $\mathbf{0.16}$ &           $0.20$ &           $0.35$ &  $\mathbf{0.78}$ &  $\mathbf{0.31}$ &  $\mathbf{0.56}$ &  $\mathbf{0.22}$ &  $\mathbf{0.83}$ &  $\mathbf{0.23}$ \\
        & SDDM (RBF) &           $0.06$ &           $0.23$ &           $0.00$ &           $0.00$ &           $0.01$ &           $0.08$ &           $0.01$ &           $0.07$ &           $0.14$ &           $0.27$ \\
& D3 &           $0.27$ &           $0.30$ &           $0.12$ &           $0.17$ &           $0.18$ &           $0.19$ &           $0.15$ &           $0.16$ &           $0.28$ &           $0.16$ \\
        & HDDDM &           $0.21$ &           $0.33$ &           $0.01$ &           $0.08$ &           $0.03$ &           $0.11$ &           $0.11$ &           $0.17$ &           $0.07$ &           $0.17$ \\
        & KSDD &           $0.09$ &           $0.06$ &           $0.08$ &           $0.07$ &           $0.13$ &           $0.07$ &           $0.12$ &           $0.08$ &           $0.16$ &           $0.07$ \\
        & MMDDDM &           $0.21$ &           $0.10$ &           $0.10$ &           $0.13$ &           $0.22$ &           $0.09$ &           $0.23$ &           $0.12$ &           $0.21$ &           $0.09$ \\
        & ShapeDDM &  $\mathbf{0.97}$ &  $\mathbf{0.10}$ &           $0.01$ &           $0.08$ &           $0.14$ &           $0.22$ &           $0.48$ &           $0.09$ &           $0.48$ &           $0.05$ \\
\hline
\multirow{7}{*}{\rotatebox[origin=c]{90}{\parbox{2cm}{Rotating\\Hyperplane}}} 
        & SDDM (MT) &           $0.24$ &           $0.42$ &           $0.00$ &           $0.04$ &           $0.00$ &           $0.00$ &           $0.18$ &           $0.23$ &           $0.20$ &           $0.25$ \\
        & SDDM (RBF) &           $0.77$ &           $0.21$ &           $0.00$ &           $0.00$ &           $0.00$ &           $0.00$ &  $\mathbf{0.70}$ &  $\mathbf{0.29}$ &  $\mathbf{0.75}$ &  $\mathbf{0.25}$ \\
& D3 &           $0.08$ &           $0.23$ &           $0.03$ &           $0.11$ &           $0.04$ &           $0.13$ &           $0.08$ &           $0.16$ &           $0.11$ &           $0.18$ \\
        & HDDDM &           $0.03$ &           $0.16$ &           $0.01$ &           $0.07$ &           $0.02$ &           $0.09$ &           $0.00$ &           $0.04$ &           $0.03$ &           $0.11$ \\
        & KSDD &           $0.04$ &           $0.10$ &           $0.06$ &           $0.10$ &           $0.07$ &           $0.11$ &           $0.04$ &           $0.08$ &           $0.07$ &           $0.10$ \\
        & MMDDDM &           $0.17$ &           $0.09$ &           $0.03$ &           $0.09$ &           $0.05$ &           $0.11$ &           $0.17$ &           $0.08$ &           $0.21$ &           $0.10$ \\
        & ShapeDDM &  $\mathbf{0.91}$ &  $\mathbf{0.15}$ &           $0.00$ &           $0.04$ &           $0.02$ &           $0.09$ &           $0.43$ &           $0.15$ &           $0.46$ &           $0.12$ \\
\hline
\multirow{7}{*}{\rotatebox[origin=c]{90}{\parbox{2cm}{STAGGER}}} 
        & SDDM (MT) &  $\mathbf{0.88}$ &  $\mathbf{0.30}$ &  $\mathbf{0.46}$ &  $\mathbf{0.47}$ &  $\mathbf{0.66}$ &  $\mathbf{0.43}$ &  $\mathbf{0.68}$ &  $\mathbf{0.33}$ &  $\mathbf{0.72}$ &  $\mathbf{0.38}$ \\
        & SDDM (RBF) &           $0.71$ &           $0.33$ &           $0.11$ &           $0.28$ &           $0.24$ &           $0.33$ &           $0.31$ &           $0.33$ &           $0.42$ &           $0.30$ \\
& D3 &           $0.27$ &           $0.31$ &           $0.17$ &           $0.18$ &           $0.25$ &           $0.18$ &           $0.18$ &           $0.17$ &           $0.28$ &           $0.18$ \\
        & HDDDM &           $0.24$ &           $0.33$ &           $0.01$ &           $0.08$ &           $0.05$ &           $0.13$ &           $0.06$ &           $0.14$ &           $0.16$ &           $0.20$ \\
        & KSDD &           $0.28$ &           $0.28$ &           $0.14$ &           $0.17$ &           $0.17$ &           $0.17$ &           $0.16$ &           $0.17$ &           $0.20$ &           $0.20$ \\
        & MMDDDM &           $0.18$ &           $0.07$ &           $0.16$ &           $0.15$ &           $0.24$ &           $0.13$ &           $0.19$ &           $0.10$ &           $0.20$ &           $0.09$ \\
        & ShapeDDM &  $\mathbf{0.94}$ &  $\mathbf{0.12}$ &           $0.06$ &           $0.17$ &           $0.16$ &           $0.22$ &           $0.24$ &           $0.25$ &           $0.28$ &           $0.24$ \\
\hline
\multirow{7}{*}{\rotatebox[origin=c]{90}{\parbox{2cm}{Nebraska\\Weather}}} 
        & SDDM (MT) &  $\mathbf{0.79}$ &  $\mathbf{0.34}$ &  $\mathbf{0.14}$ &  $\mathbf{0.30}$ &  $\mathbf{0.49}$ &  $\mathbf{0.43}$ &  $\mathbf{0.66}$ &  $\mathbf{0.31}$ &  $\mathbf{0.91}$ &  $\mathbf{0.13}$ \\
        & SDDM (RBF) &           $0.07$ &           $0.25$ &           $0.03$ &           $0.16$ &           $0.10$ &           $0.26$ &           $0.49$ &           $0.39$ &           $0.80$ &           $0.21$ \\
& D3 &           $0.06$ &           $0.18$ &  $\mathbf{0.12}$ &  $\mathbf{0.17}$ &           $0.18$ &           $0.18$ &           $0.23$ &           $0.17$ &           $0.24$ &           $0.16$ \\
        & HDDDM &           $0.01$ &           $0.08$ &           $0.00$ &           $0.04$ &           $0.02$ &           $0.10$ &           $0.00$ &           $0.00$ &           $0.09$ &           $0.19$ \\
        & KSDD &           $0.06$ &           $0.05$ &           $0.07$ &           $0.06$ &           $0.11$ &           $0.05$ &           $0.10$ &           $0.04$ &           $0.12$ &           $0.04$ \\
        & MMDDDM &           $0.05$ &           $0.10$ &           $0.06$ &           $0.12$ &           $0.12$ &           $0.12$ &           $0.18$ &           $0.09$ &           $0.24$ &           $0.10$ \\
        & ShapeDDM &           $0.55$ &           $0.48$ &           $0.01$ &           $0.07$ &           $0.03$ &           $0.12$ &           $0.17$ &           $0.23$ &           $0.18$ &           $0.24$ \\
\bottomrule
\end{tabular}

    \label{tbl:main_results}
\end{table}

We used standard drift benchmarks: ``Rotating Hyperplane''~\cite{skmultiflow}, 
``STAGGER''~\cite{LearningWithDriftDetection}
and ``RandomRBF''~\cite{skmultiflow}. 
And the following real world datasets: ``Electricity Market Prices''~\cite{electricitymarketdata}, ``Forest Covertype''~\cite{forestcovertypedataset}, and ``Nebraska Weather''~\cite{weatherdataset}. If a dataset provides a label, it is considered as an additional feature, thereby real drift is turned into distributional drift.
To control the concepts obtained from the real world datasets,
we split the datasets into time windows -- which then correspond to the concepts -- and then selecting samples randomly from those, thereby removing any unwanted drift during one concept. 
Using these concepts (either from the artificial or the real world dataset) we create data streams with a length of $1{,}600$ samples each, which consist of three segments (Warmup, Drift, and cOol down); the first (W), and last segment (O) contain $500+\Delta$ and $600-\Delta$ samples, respectively, and continue the neighbouring concept from the middle segment (D), where $\Delta$ is a randomly chosen number in range between $0$ and $100$. The middle segment (D) is $500$ samples long and contains 2 or 3 concepts in the ordering: AB, ABA, or ABC. In case AB, both concepts contain equally many samples. In cases ABA and ABC, the length of concept B is varied ($|B| \in \{25,50,75,100\}$) and the other two concepts are of equal length. Every stream is normalized before running the experiments.

We compare our proposed method spectral drift detection (with MomentTree-kernel) (SDDM (MT)) against the shape based drift detector (ShapeDDM)~\cite{SSCI}, which matches the drift magnitude against a certain shape, the usual MMD drift detection (MMDDDM)~\cite{MMD,fail}, which applies the MMD two sample test at every time point, 
the Hellinger distance drift detection method (HDDDM)~\cite{HDDM}, which computes the Hellinger distance using bins of the marignal distribution, Kolmogorov-Smirnov Windowing (KS-Win)~\cite{KSWIN}, which applies the Kolmogorov-Smirnov test feature wise, and D3~\cite{D3}, which is one of the state of the art drift detectors for unsupervised drift detection. 
We also consider SDDM with Gauss-kernel (SDDM (RBF)). As far as given we used default parameters for all methods; if not present or determined, we choose the optimal parameters with respect to $50$ independent runs. In particular, for all methods except SDDM we selected a shift of the detected drift points to optimally align them with the true drift events to compensate for potential, inherent delays.

We measure the quality of a method using the $\beta$-score~\cite{SSCI}, which is defined as $\frac{\text{tp}}{\text{p}+\beta \text{fp}}$ where $\text{p}$ is the number of all drifts, $\text{tp}$ the number of true positives, and $\text{fp}$ the number of false positives. Here, the closest drift alert to the actual drift event is considered as a true positive if it is still within range of the maximal delay, otherwise it is considered as a false positive. Every other detection is a false positive. We choose the maximal delay as halve the length of concept B (in the cases ABA and ABC) so that is is clear which drift is detected. Part of the results are shown in Fig.~\ref{fig:results_weather} and Table~\ref{tbl:main_results}, significantly best results ($t$-test, $p=0.001$) are marked bold face (additional tables and figures are provided in the supplement).

\begin{toappendix}

In addition to the results already presented in Section~\ref{sec:exp} we will provide additional analysis here. Table~\ref{tab:alarms} summarizes the median number of alerts and thus provides an upper bound on the specificity of the methods. Table~\ref{tab:delay} documents the median delay between a drift event and the closest detected drift, i.e. if choosing this as max delay, 50\% of all events are found. Similar information is provided in Fig.~\ref{fig:delay}, which shows the number of detected drifts vs. maximal delay. Table~\ref{tab:all} and Table~\ref{tab:short} show the $\beta$-score ($\beta = 0.5$) for all split length with varying choices of maximal delay. 

\begin{table}[]
    \centering
    \caption{Results over 150 runs. Table shows median number of alarms.}
\begin{tabular}{ll|@{\quad}r@{\quad}r@{\quad}r@{\quad}r@{\quad}r@{\quad}r@{\quad}r@{\quad}r@{\quad}r}
\toprule
DS & Method &  AB & \multicolumn{4}{c}{ABA} &  \multicolumn{4}{c}{ABC} \\
&& -- & 25 & 50 & 75 & 100 & 25 & 50 & 75 & 100 \\
\midrule
\multirow{7}{*}{\rotatebox[origin=c]{90}{\parbox{2cm}{Electricity\\Market Prices}}} & D3 &          2 &                2 &                3 &                3 &                 3 &                3 &                3 &                3 &                 3 \\
        & HDDDM &          1 &                0 &                0 &                0 &                 0 &                1 &                1 &                1 &                 1 \\
        & KSDD &         12 &               11 &               11 &               11 &                11 &               12 &               12 &               12 &                12 \\
        & MMDDDM &         12 &                7 &                9 &               13 &                13 &               16 &               16 &               17 &                20 \\
        & SDDM(MT) &          1 &                0 &                2 &                2 &                 2 &                2 &                2 &                2 &                 2 \\
        & SDDM(RBF) &          1 &                0 &                0 &                0 &                 0 &                2 &                2 &                2 &                 2 \\
        & ShapeDDM &          1 &                0 &                0 &                0 &                 1 &                1 &                1 &                1 &                 1 \\
\hline\multirow{7}{*}{\rotatebox[origin=c]{90}{\parbox{2cm}{Forest\\Covertype}}} & D3 &          3 &                2 &                3 &                3 &                 3 &                3 &                3 &                3 &                 3 \\
        & HDDDM &          2 &                1 &                1 &                1 &                 1 &                2 &                1 &                1 &                 1 \\
        & KSDD &         31 &               32 &               35 &               36 &                36 &               32 &               31 &               30 &                31 \\
        & MMDDDM &         12 &               14 &               11 &               11 &                12 &               12 &               15 &               16 &                17 \\
        & SDDM(MT) &          1 &                3 &                2 &                3 &                 3 &                2 &                2 &                2 &                 2 \\
        & SDDM(RBF) &          0 &                0 &                0 &                0 &                 0 &                0 &                0 &                0 &                 0 \\
        & ShapeDDM &          1 &                0 &                1 &                1 &                 1 &                1 &                1 &                1 &                 1 \\
\hline\multirow{7}{*}{\rotatebox[origin=c]{90}{\parbox{2cm}{Random\\RBF}}} & D3 &          3 &                2 &                2 &                3 &                 3 &                3 &                3 &                3 &                 3 \\
        & HDDDM &          2 &                0 &                0 &                0 &                 1 &                2 &                1 &                1 &                 1 \\
        & KSDD &         16 &               17 &               17 &               17 &                17 &               15 &               17 &               17 &                18 \\
        & MMDDDM &          9 &               11 &               16 &               15 &                16 &               10 &               13 &               16 &                17 \\
        & SDDM(MT) &          1 &                0 &                2 &                2 &                 2 &                1 &                2 &                2 &                 2 \\
        & SDDM(RBF) &          0 &                0 &                0 &                0 &                 0 &                0 &                0 &                0 &                 0 \\
        & ShapeDDM &          1 &                0 &                1 &                1 &                 1 &                1 &                1 &                1 &                 1 \\
\hline\multirow{7}{*}{\rotatebox[origin=c]{90}{\parbox{2cm}{Rotating\\Hyperplane}}} & D3 &          1 &                1 &                1 &                1 &                 1 &                2 &                2 &                2 &                 2 \\
        & HDDDM &          0 &                0 &                0 &                0 &                 0 &                0 &                0 &                0 &                 0 \\
        & KSDD &          7 &                7 &                7 &                7 &                 7 &                7 &                7 &                7 &                 7 \\
        & MMDDDM &         12 &                7 &                5 &                6 &                 8 &               13 &               14 &               17 &                17 \\
        & SDDM(MT) &          0 &                0 &                0 &                0 &                 0 &                0 &                0 &                0 &                 0 \\
        & SDDM(RBF) &          2 &                0 &                0 &                0 &                 0 &                2 &                2 &                2 &                 2 \\
        & ShapeDDM &          1 &                0 &                0 &                0 &                 0 &                1 &                1 &                1 &                 1 \\
\hline\multirow{7}{*}{\rotatebox[origin=c]{90}{\parbox{2cm}{STAGGER}}} & D3 &          2 &                3 &                3 &                3 &                 3 &                3 &                2 &                2 &                 3 \\
        & HDDDM &          2 &                1 &                1 &                1 &                 1 &                1 &                2 &                1 &                 1 \\
        & KSDD &          3 &                3 &                3 &                3 &                 4 &                3 &                3 &                3 &                 4 \\
        & MMDDDM &         11 &                8 &               16 &               18 &                21 &               10 &               18 &               16 &                22 \\
        & SDDM(MT) &          1 &                1 &                2 &                2 &                 2 &                2 &                2 &                2 &                 2 \\
        & SDDM(RBF) &          1 &                0 &                0 &                2 &                 2 &                1 &                2 &                2 &                 2 \\
        & ShapeDDM &          1 &                0 &                1 &                1 &                 2 &                1 &                1 &                1 &                 1 \\
\hline\multirow{7}{*}{\rotatebox[origin=c]{90}{\parbox{2cm}{Nebraska\\Weather}}} & D3 &          2 &                2 &                3 &                3 &                 3 &                3 &                3 &                3 &                 3 \\
        & HDDDM &          0 &                0 &                0 &                0 &                 0 &                0 &                1 &                1 &                 1 \\
        & KSDD &         22 &               22 &               23 &               23 &                24 &               24 &               25 &               25 &                26 \\
        & MMDDDM &          8 &                7 &               14 &               14 &                17 &               17 &               15 &               15 &                16 \\
        & SDDM(MT) &          1 &                0 &                2 &                2 &                 2 &                2 &                2 &                2 &                 2 \\
        & SDDM(RBF) &          0 &                0 &                0 &                0 &                 1 &                1 &                2 &                3 &                 2 \\
        & ShapeDDM &          1 &                0 &                0 &                1 &                 1 &                1 &                1 &                0 &                 0 \\
\bottomrule
\end{tabular}
    \label{tab:alarms}
\end{table}
\begin{table}[]
    \centering
    \caption{Results over 150 runs. Table shows median distance between drift event and closes alarm.}
\begin{tabular}{ll|@{\quad}r@{\quad}r@{\quad}r@{\quad}r@{\quad}r@{\quad}r@{\quad}r@{\quad}r@{\quad}r}
\toprule
DS & Method &  AB & \multicolumn{4}{c}{ABA} &  \multicolumn{4}{c}{ABC} \\
&& -- & 25 & 50 & 75 & 100 & 25 & 50 & 75 & 100 \\
\midrule
\multirow{7}{*}{\rotatebox[origin=c]{90}{\parbox{2cm}{Electricity\\Market Prices}}} & D3 &         67 &              85 &              66 &              53 &               50 &              16 &               27 &              37 &                50 \\
        & HDDDM &         28 &             263 &             264 &             230 &              397 &              31 &               47 &              56 &                67 \\
        & KSDD &         21 &              23 &              25 &              37 &               37 &              12 &               16 &              11 &                 8 \\
        & MMDDDM &         10 &             134 &              29 &              26 &               26 &               6 &                5 &               6 &                 6 \\
        & SDDM(MT) &          2 &               3 &               3 &               2 &                3 &               1 &                1 &               0 &                 1 \\
        & SDDM(RBF) &          5 &             203 &              57 &              37 &               51 &               3 &                3 &               3 &                 3 \\
        & ShapeDDM &          9 &             210 &             144 &              85 &               75 &              15 &               25 &              37 &                50 \\
\hline\multirow{7}{*}{\rotatebox[origin=c]{90}{\parbox{2cm}{Forest\\Covertype}}} & D3 &          6 &              20 &              30 &              27 &                7 &              15 &               27 &              37 &                18 \\
        & HDDDM &         31 &              77 &              69 &              86 &              100 &              29 &               27 &              37 &                50 \\
        & KSDD &          6 &              11 &               9 &               4 &                3 &               7 &                9 &               9 &                 7 \\
        & MMDDDM &          2 &              13 &              14 &              10 &               11 &               9 &                7 &               3 &                 3 \\
        & SDDM(MT) &          2 &               0 &               1 &               0 &                1 &              12 &                2 &               1 &                 2 \\
        & SDDM(RBF) &          3 &                -- &                -- &                -- &                 -- &                -- &               31 &              37 &                50 \\
        & ShapeDDM &          4 &              45 &              35 &              42 &               50 &              12 &               25 &              37 &                50 \\
\hline\multirow{7}{*}{\rotatebox[origin=c]{90}{\parbox{2cm}{Random\\RBF}}} & D3 &         20 &              86 &              37 &              37 &               50 &              27 &               27 &              37 &                50 \\
        & HDDDM &         34 &             345 &             362 &             224 &              125 &              39 &               85 &             112 &               113 \\
        & KSDD &          8 &              13 &              20 &              15 &               16 &              12 &               13 &              14 &                 9 \\
        & MMDDDM &          2 &              17 &               9 &               6 &                3 &               8 &               10 &               4 &                 3 \\
        & SDDM(MT) &          2 &              12 &               2 &               1 &                2 &              12 &                3 &               1 &                 2 \\
        & SDDM(RBF) &          1 &                -- &              27 &              37 &               32 &              12 &               25 &              37 &                44 \\
        & ShapeDDM &          4 &             266 &              56 &              50 &               54 &              12 &               25 &              37 &                50 \\
\hline\multirow{7}{*}{\rotatebox[origin=c]{90}{\parbox{2cm}{Rotating\\Hyperplane}}} & D3 &        169 &             176 &             200 &             197 &              212 &             135 &              171 &              78 &                85 \\
        & HDDDM &        321 &             373 &             344 &             385 &              286 &             260 &              348 &             342 &               334 \\
        & KSDD &         83 &              50 &              73 &              64 &               60 &              58 &               74 &              53 &                55 \\
        & MMDDDM &          2 &             165 &             157 &             154 &              133 &               9 &                8 &               2 &                 2 \\
        & SDDM(MT) &          3 &             558 &             533 &             398 &              447 &              12 &               25 &              37 &                50 \\
        & SDDM(RBF) &          2 &             111 &             679 &             182 &              777 &               3 &                4 &               4 &                 5 \\
        & ShapeDDM &          4 &             388 &             284 &             301 &              292 &              12 &               25 &              37 &                50 \\
\hline\multirow{7}{*}{\rotatebox[origin=c]{90}{\parbox{2cm}{STAGGER}}} & D3 &         30 &              23 &              28 &              37 &               50 &              20 &               29 &              37 &                50 \\
        & HDDDM &         36 &             157 &             102 &             104 &              311 &              72 &               66 &              75 &                80 \\
        & KSDD &         16 &              36 &              41 &              39 &               48 &              22 &               31 &              38 &                50 \\
        & MMDDDM &          4 &              12 &               8 &               8 &                6 &              10 &               11 &               7 &                 6 \\
        & SDDM(MT) &          1 &               1 &               1 &               0 &                1 &               2 &                2 &               1 &                 2 \\
        & SDDM(RBF) &          2 &              12 &               4 &               1 &                2 &              12 &               16 &               6 &                12 \\
        & ShapeDDM &          4 &              26 &              48 &             450 &               51 &              12 &               25 &              38 &                50 \\
\hline\multirow{7}{*}{\rotatebox[origin=c]{90}{\parbox{2cm}{Nebraska\\Weather}}} & D3 &        197 &              76 &              34 &              38 &               50 &              16 &               28 &              37 &                50 \\
        & HDDDM &        311 &             338 &             308 &             235 &              253 &             178 &               74 &              71 &                80 \\
        & KSDD &         15 &              12 &              18 &              14 &               11 &              10 &               12 &               7 &                 5 \\
        & MMDDDM &         57 &              88 &              21 &              17 &               14 &               4 &                6 &               8 &                 5 \\
        & SDDM(MT) &          4 &              13 &               3 &               3 &                4 &               4 &                2 &               1 &                 2 \\
        & SDDM(RBF) &          9 &               1 &              25 &              26 &                9 &               5 &                2 &               1 &                 2 \\
        & ShapeDDM &         13 &             278 &              88 &              76 &               73 &              27 &               42 &              50 &                64 \\
\bottomrule
\end{tabular}
    \label{tab:delay}
\end{table}
\begin{figure}
    \centering
    \includegraphics[width=\textwidth]{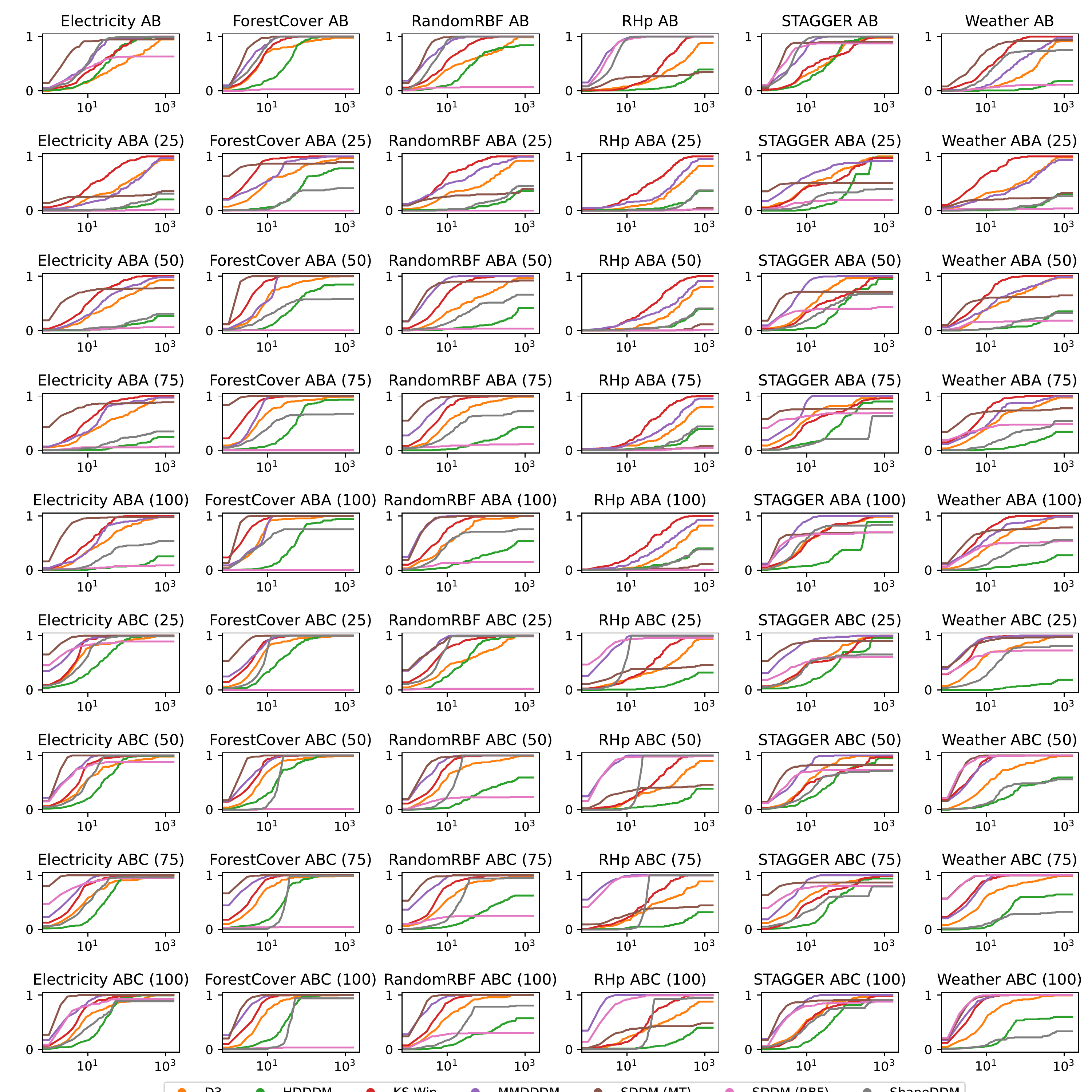}
    \caption{Rate of detected drifts ($y$-axis) for different choices of max delay ($x$-axis).}
    \label{fig:delay}
\end{figure}

\begin{landscape}
\begin{longtable}{ll|l@{$\pm$}l@{\;\;}l@{$\pm$}l@{\;\;}l@{$\pm$}l@{\;\;}l@{$\pm$}l@{\;\;}l@{$\pm$}l@{\;\;}l@{$\pm$}l@{\;\;}l@{$\pm$}l@{\;\;}l@{$\pm$}l@{\;\;}l@{$\pm$}l}
\caption{Results over 150 runs. Table shows mean $\beta$-score ($\beta = 0.5$) and standard deviation. Numbers below stream types are B-length and maximal delay.}
\label{tab:all}
\endfirsthead
\toprule
DS & Method & \multicolumn{2}{c}{AB} &  \multicolumn{8}{c}{ABA} & \multicolumn{8}{c}{ABC} \\
&& \multicolumn{2}{c}{--/50}&\multicolumn{2}{c}{25/12}&\multicolumn{2}{c}{50/25}&\multicolumn{2}{c}{75/36}&\multicolumn{2}{c}{100/50}&\multicolumn{2}{c}{25/12}&\multicolumn{2}{c}{50/25}&\multicolumn{2}{c}{75/36}&\multicolumn{2}{c}{100/50} \\
\midrule
\multirow{7}{*}{\rotatebox[origin=c]{90}{\parbox{2cm}{Electricity\\Market Prices}}} & D3 &           $0.23$ &           $0.30$ &           $0.09$ &           $0.16$ &           $0.13$ &           $0.17$ &           $0.20$ &           $0.20$ &           $0.29$ &           $0.23$ &           $0.29$ &           $0.15$ &           $0.28$ &           $0.16$ &           $0.35$ &           $0.21$ &           $0.41$ &           $0.25$ \\
        & HDDDM &           $0.65$ &           $0.44$ &           $0.00$ &           $0.00$ &           $0.01$ &           $0.07$ &           $0.02$ &           $0.10$ &           $0.03$ &           $0.13$ &           $0.18$ &           $0.23$ &           $0.25$ &           $0.24$ &           $0.31$ &           $0.24$ &           $0.41$ &           $0.19$ \\
        & KSDD &           $0.12$ &           $0.08$ &           $0.07$ &           $0.08$ &           $0.14$ &           $0.10$ &           $0.17$ &           $0.09$ &           $0.21$ &           $0.09$ &           $0.18$ &           $0.09$ &           $0.21$ &           $0.10$ &           $0.23$ &           $0.09$ &           $0.25$ &           $0.09$ \\
        & MMDDDM &           $0.17$ &           $0.09$ &           $0.02$ &           $0.06$ &           $0.14$ &           $0.16$ &           $0.19$ &           $0.13$ &           $0.19$ &           $0.13$ &           $0.17$ &           $0.08$ &           $0.24$ &           $0.09$ &           $0.23$ &           $0.08$ &           $0.19$ &           $0.06$ \\
        & SDDM(MT) &  $\mathbf{0.89}$ &  $\mathbf{0.25}$ &  $\mathbf{0.20}$ &  $\mathbf{0.36}$ &  $\mathbf{0.65}$ &  $\mathbf{0.40}$ &  $\mathbf{0.75}$ &  $\mathbf{0.35}$ &  $\mathbf{0.88}$ &  $\mathbf{0.21}$ &  $\mathbf{0.86}$ &  $\mathbf{0.22}$ &  $\mathbf{0.95}$ &  $\mathbf{0.10}$ &  $\mathbf{0.97}$ &  $\mathbf{0.08}$ &  $\mathbf{0.97}$ &  $\mathbf{0.07}$ \\
        & SDDM(RBF) &           $0.55$ &           $0.45$ &           $0.00$ &           $0.00$ &           $0.01$ &           $0.07$ &           $0.04$ &           $0.16$ &           $0.05$ &           $0.19$ &           $0.64$ &           $0.36$ &           $0.67$ &           $0.34$ &           $0.79$ &           $0.28$ &           $0.78$ &           $0.29$ \\
        & ShapeDDM &  $\mathbf{0.93}$ &  $\mathbf{0.18}$ &           $0.00$ &           $0.04$ &           $0.03$ &           $0.11$ &           $0.07$ &           $0.17$ &           $0.19$ &           $0.24$ &           $0.37$ &           $0.21$ &           $0.47$ &           $0.09$ &           $0.46$ &           $0.11$ &           $0.43$ &           $0.16$ \\
\hline
\multirow{7}{*}{\rotatebox[origin=c]{90}{\parbox{2cm}{Forest\\Covertype}}} & D3 &           $0.46$ &           $0.27$ &           $0.21$ &           $0.18$ &           $0.26$ &           $0.15$ &           $0.49$ &           $0.28$ &           $0.62$ &           $0.23$ &           $0.29$ &           $0.15$ &           $0.31$ &           $0.12$ &           $0.41$ &           $0.22$ &           $0.54$ &           $0.23$ \\
        & HDDDM &           $0.57$ &           $0.38$ &           $0.03$ &           $0.11$ &           $0.12$ &           $0.21$ &           $0.18$ &           $0.23$ &           $0.21$ &           $0.23$ &           $0.17$ &           $0.21$ &           $0.33$ &           $0.20$ &           $0.36$ &           $0.19$ &           $0.36$ &           $0.18$ \\
        & KSDD &           $0.07$ &           $0.01$ &           $0.07$ &           $0.03$ &           $0.09$ &           $0.03$ &           $0.10$ &           $0.02$ &           $0.10$ &           $0.02$ &           $0.09$ &           $0.04$ &           $0.10$ &           $0.03$ &           $0.12$ &           $0.03$ &           $0.12$ &           $0.02$ \\
        & MMDDDM &           $0.17$ &           $0.09$ &           $0.11$ &           $0.12$ &           $0.29$ &           $0.15$ &           $0.33$ &           $0.15$ &           $0.32$ &           $0.15$ &           $0.19$ &           $0.13$ &           $0.21$ &           $0.08$ &           $0.23$ &           $0.08$ &           $0.24$ &           $0.09$ \\
        & SDDM(MT) &           $0.92$ &           $0.15$ &  $\mathbf{0.69}$ &  $\mathbf{0.32}$ &  $\mathbf{0.89}$ &  $\mathbf{0.13}$ &  $\mathbf{0.86}$ &  $\mathbf{0.14}$ &  $\mathbf{0.86}$ &  $\mathbf{0.15}$ &  $\mathbf{0.66}$ &  $\mathbf{0.24}$ &  $\mathbf{0.87}$ &  $\mathbf{0.19}$ &  $\mathbf{0.91}$ &  $\mathbf{0.13}$ &  $\mathbf{0.93}$ &  $\mathbf{0.12}$ \\
        & SDDM(RBF) &           $0.03$ &           $0.16$ &           $0.00$ &           $0.00$ &           $0.00$ &           $0.00$ &           $0.00$ &           $0.00$ &           $0.00$ &           $0.00$ &           $0.00$ &           $0.00$ &           $0.01$ &           $0.06$ &           $0.03$ &           $0.13$ &           $0.02$ &           $0.13$ \\
        & ShapeDDM &  $\mathbf{0.98}$ &  $\mathbf{0.07}$ &           $0.02$ &           $0.10$ &           $0.22$ &           $0.25$ &           $0.32$ &           $0.24$ &           $0.36$ &           $0.21$ &           $0.47$ &           $0.11$ &           $0.47$ &           $0.11$ &           $0.48$ &           $0.08$ &           $0.45$ &           $0.15$ \\
\hline\multirow{7}{*}{\rotatebox[origin=c]{90}{\parbox{2cm}{Random\\RBF}}} & D3 &           $0.31$ &           $0.30$ &           $0.12$ &           $0.17$ &           $0.18$ &           $0.19$ &           $0.31$ &           $0.21$ &           $0.42$ &           $0.23$ &           $0.15$ &           $0.16$ &           $0.28$ &           $0.16$ &           $0.37$ &           $0.22$ &           $0.45$ &           $0.24$ \\
        & HDDDM &           $0.35$ &           $0.36$ &           $0.01$ &           $0.08$ &           $0.03$ &           $0.11$ &           $0.04$ &           $0.14$ &           $0.10$ &           $0.20$ &           $0.11$ &           $0.17$ &           $0.07$ &           $0.17$ &           $0.10$ &           $0.19$ &           $0.13$ &           $0.21$ \\
        & KSDD &           $0.11$ &           $0.04$ &           $0.08$ &           $0.07$ &           $0.13$ &           $0.07$ &           $0.18$ &           $0.07$ &           $0.18$ &           $0.06$ &           $0.12$ &           $0.08$ &           $0.16$ &           $0.07$ &           $0.16$ &           $0.07$ &           $0.19$ &           $0.05$ \\
        & MMDDDM &           $0.21$ &           $0.10$ &           $0.10$ &           $0.13$ &           $0.22$ &           $0.09$ &           $0.26$ &           $0.10$ &           $0.24$ &           $0.11$ &           $0.23$ &           $0.12$ &           $0.21$ &           $0.09$ &           $0.23$ &           $0.09$ &           $0.23$ &           $0.09$ \\
        & SDDM(MT) &  $\mathbf{0.93}$ &  $\mathbf{0.16}$ &           $0.20$ &           $0.35$ &  $\mathbf{0.78}$ &  $\mathbf{0.31}$ &  $\mathbf{0.90}$ &  $\mathbf{0.14}$ &  $\mathbf{0.92}$ &  $\mathbf{0.13}$ &  $\mathbf{0.56}$ &  $\mathbf{0.22}$ &  $\mathbf{0.83}$ &  $\mathbf{0.23}$ &  $\mathbf{0.89}$ &  $\mathbf{0.18}$ &  $\mathbf{0.91}$ &  $\mathbf{0.16}$ \\
        & SDDM(RBF) &           $0.06$ &           $0.23$ &           $0.00$ &           $0.00$ &           $0.01$ &           $0.08$ &           $0.07$ &           $0.22$ &           $0.11$ &           $0.27$ &           $0.01$ &           $0.07$ &           $0.14$ &           $0.27$ &           $0.17$ &           $0.31$ &           $0.21$ &           $0.35$ \\
        & ShapeDDM &  $\mathbf{0.97}$ &  $\mathbf{0.10}$ &           $0.01$ &           $0.08$ &           $0.14$ &           $0.22$ &           $0.30$ &           $0.24$ &           $0.34$ &           $0.22$ &           $0.48$ &           $0.09$ &           $0.48$ &           $0.05$ &           $0.45$ &           $0.13$ &           $0.38$ &           $0.20$ \\
\pagebreak
\hline\multirow{7}{*}{\rotatebox[origin=c]{90}{\parbox{2cm}{Rotating\\Hyperplane}}} & D3 &           $0.14$ &           $0.30$ &           $0.03$ &           $0.11$ &           $0.04$ &           $0.13$ &           $0.07$ &           $0.17$ &           $0.07$ &           $0.16$ &           $0.08$ &           $0.16$ &           $0.11$ &           $0.18$ &           $0.17$ &           $0.21$ &           $0.19$ &           $0.22$ \\
        & HDDDM &           $0.03$ &           $0.18$ &           $0.01$ &           $0.07$ &           $0.02$ &           $0.09$ &           $0.02$ &           $0.10$ &           $0.04$ &           $0.13$ &           $0.00$ &           $0.04$ &           $0.03$ &           $0.11$ &           $0.03$ &           $0.11$ &           $0.03$ &           $0.11$ \\
        & KSDD &           $0.08$ &           $0.13$ &           $0.06$ &           $0.10$ &           $0.07$ &           $0.11$ &           $0.11$ &           $0.12$ &  $\mathbf{0.15}$ &  $\mathbf{0.14}$ &           $0.04$ &           $0.08$ &           $0.07$ &           $0.10$ &           $0.13$ &           $0.13$ &           $0.17$ &           $0.15$ \\
        & MMDDDM &           $0.17$ &           $0.09$ &           $0.03$ &           $0.09$ &           $0.05$ &           $0.11$ &           $0.07$ &           $0.14$ &           $0.08$ &           $0.12$ &           $0.17$ &           $0.08$ &           $0.21$ &           $0.10$ &           $0.22$ &           $0.08$ &           $0.22$ &           $0.08$ \\
        & SDDM(MT) &           $0.25$ &           $0.42$ &           $0.00$ &           $0.04$ &           $0.00$ &           $0.00$ &           $0.00$ &           $0.03$ &           $0.01$ &           $0.07$ &           $0.18$ &           $0.23$ &           $0.20$ &           $0.25$ &           $0.20$ &           $0.27$ &           $0.23$ &           $0.29$ \\
        & SDDM(RBF) &           $0.78$ &           $0.19$ &           $0.00$ &           $0.00$ &           $0.00$ &           $0.00$ &           $0.00$ &           $0.00$ &           $0.00$ &           $0.00$ &  $\mathbf{0.70}$ &  $\mathbf{0.29}$ &  $\mathbf{0.75}$ &  $\mathbf{0.25}$ &  $\mathbf{0.80}$ &  $\mathbf{0.21}$ &  $\mathbf{0.80}$ &  $\mathbf{0.22}$ \\
        & ShapeDDM &  $\mathbf{0.91}$ &  $\mathbf{0.15}$ &           $0.00$ &           $0.04$ &           $0.02$ &           $0.09$ &           $0.02$ &           $0.09$ &           $0.03$ &           $0.11$ &           $0.43$ &           $0.15$ &           $0.46$ &           $0.12$ &           $0.46$ &           $0.10$ &           $0.43$ &           $0.15$ \\
\hline\multirow{7}{*}{\rotatebox[origin=c]{90}{\parbox{2cm}{STAGGER}}} & D3 &           $0.34$ &           $0.31$ &           $0.17$ &           $0.18$ &           $0.25$ &           $0.18$ &           $0.36$ &           $0.23$ &           $0.42$ &           $0.26$ &           $0.18$ &           $0.17$ &           $0.28$ &           $0.18$ &           $0.36$ &           $0.23$ &           $0.43$ &           $0.25$ \\
        & HDDDM &           $0.41$ &           $0.36$ &           $0.01$ &           $0.08$ &           $0.05$ &           $0.13$ &           $0.15$ &           $0.23$ &           $0.12$ &           $0.21$ &           $0.06$ &           $0.14$ &           $0.16$ &           $0.20$ &           $0.21$ &           $0.23$ &           $0.26$ &           $0.23$ \\
        & KSDD &           $0.31$ &           $0.27$ &           $0.14$ &           $0.17$ &           $0.17$ &           $0.17$ &           $0.27$ &           $0.23$ &           $0.37$ &           $0.23$ &           $0.16$ &           $0.17$ &           $0.20$ &           $0.20$ &           $0.29$ &           $0.23$ &           $0.33$ &           $0.23$ \\
        & MMDDDM &           $0.18$ &           $0.07$ &           $0.16$ &           $0.15$ &           $0.24$ &           $0.13$ &           $0.23$ &           $0.09$ &           $0.18$ &           $0.06$ &           $0.19$ &           $0.10$ &           $0.20$ &           $0.09$ &           $0.24$ &           $0.10$ &           $0.18$ &           $0.05$ \\
        & SDDM(MT) &  $\mathbf{0.88}$ &  $\mathbf{0.30}$ &  $\mathbf{0.46}$ &  $\mathbf{0.47}$ &  $\mathbf{0.66}$ &  $\mathbf{0.43}$ &  $\mathbf{0.73}$ &  $\mathbf{0.42}$ &  $\mathbf{0.68}$ &  $\mathbf{0.46}$ &  $\mathbf{0.68}$ &  $\mathbf{0.33}$ &  $\mathbf{0.72}$ &  $\mathbf{0.38}$ &  $\mathbf{0.82}$ &  $\mathbf{0.35}$ &  $\mathbf{0.86}$ &  $\mathbf{0.30}$ \\
        & SDDM(RBF) &           $0.71$ &           $0.33$ &           $0.11$ &           $0.28$ &           $0.24$ &           $0.33$ &           $0.50$ &           $0.38$ &           $0.49$ &           $0.37$ &           $0.31$ &           $0.33$ &           $0.42$ &           $0.30$ &           $0.50$ &           $0.32$ &           $0.53$ &           $0.29$ \\
        & ShapeDDM &  $\mathbf{0.94}$ &  $\mathbf{0.12}$ &           $0.06$ &           $0.17$ &           $0.16$ &           $0.22$ &           $0.09$ &           $0.18$ &           $0.35$ &           $0.17$ &           $0.24$ &           $0.25$ &           $0.28$ &           $0.24$ &           $0.27$ &           $0.22$ &           $0.33$ &           $0.20$ \\
\hline\multirow{7}{*}{\rotatebox[origin=c]{90}{\parbox{2cm}{Nebraska\\Weather}}} & D3 &           $0.10$ &           $0.22$ &  $\mathbf{0.12}$ &  $\mathbf{0.17}$ &           $0.18$ &           $0.18$ &           $0.26$ &           $0.19$ &           $0.37$ &           $0.28$ &           $0.23$ &           $0.17$ &           $0.24$ &           $0.16$ &           $0.32$ &           $0.21$ &           $0.45$ &           $0.25$ \\
        & HDDDM &           $0.01$ &           $0.08$ &           $0.00$ &           $0.04$ &           $0.02$ &           $0.10$ &           $0.03$ &           $0.11$ &           $0.04$ &           $0.13$ &           $0.00$ &           $0.00$ &           $0.09$ &           $0.19$ &           $0.17$ &           $0.23$ &           $0.23$ &           $0.25$ \\
        & KSDD &           $0.08$ &           $0.03$ &           $0.07$ &           $0.06$ &           $0.11$ &           $0.05$ &           $0.13$ &           $0.05$ &           $0.14$ &           $0.04$ &           $0.10$ &           $0.04$ &           $0.12$ &           $0.04$ &           $0.13$ &           $0.03$ &           $0.14$ &           $0.03$ \\
        & MMDDDM &           $0.11$ &           $0.17$ &           $0.06$ &           $0.12$ &           $0.12$ &           $0.12$ &           $0.20$ &           $0.13$ &           $0.18$ &           $0.11$ &           $0.18$ &           $0.09$ &           $0.24$ &           $0.10$ &           $0.25$ &           $0.09$ &           $0.24$ &           $0.11$ \\
        & SDDM(MT) &  $\mathbf{0.82}$ &  $\mathbf{0.30}$ &  $\mathbf{0.14}$ &  $\mathbf{0.30}$ &  $\mathbf{0.49}$ &  $\mathbf{0.43}$ &  $\mathbf{0.63}$ &  $\mathbf{0.42}$ &  $\mathbf{0.64}$ &  $\mathbf{0.41}$ &  $\mathbf{0.66}$ &  $\mathbf{0.31}$ &  $\mathbf{0.91}$ &  $\mathbf{0.13}$ &  $\mathbf{0.91}$ &  $\mathbf{0.15}$ &  $\mathbf{0.91}$ &  $\mathbf{0.14}$ \\
        & SDDM(RBF) &           $0.09$ &           $0.27$ &           $0.03$ &           $0.16$ &           $0.10$ &           $0.26$ &           $0.33$ &           $0.39$ &           $0.40$ &           $0.43$ &           $0.49$ &           $0.39$ &           $0.80$ &           $0.21$ &  $\mathbf{0.87}$ &  $\mathbf{0.14}$ &  $\mathbf{0.87}$ &  $\mathbf{0.14}$ \\
        & ShapeDDM &           $0.66$ &           $0.44$ &           $0.01$ &           $0.07$ &           $0.03$ &           $0.12$ &           $0.12$ &           $0.21$ &           $0.19$ &           $0.24$ &           $0.17$ &           $0.23$ &           $0.18$ &           $0.24$ &           $0.13$ &           $0.22$ &           $0.10$ &           $0.20$ \\
\bottomrule
\end{longtable}
\pagebreak
\begin{longtable}{ll|l@{$\pm$}l@{\;\;}l@{$\pm$}l@{\;\;}l@{$\pm$}l@{\;\;}l@{$\pm$}l@{\;\;}l@{$\pm$}l@{\;\;}l@{$\pm$}l@{\;\;}l@{$\pm$}l@{\;\;}l@{$\pm$}l@{\;\;}l@{$\pm$}l}
\caption{Results over 150 runs. Table shows mean $\beta$-score ($\beta = 0.5$) and standard deviation. Numbers below stream types are B-length. Maximal delay is $12$.}
\label{tab:short}
\endfirsthead
\toprule
DS & Method & \multicolumn{2}{c}{AB} &  \multicolumn{8}{c}{ABA} & \multicolumn{8}{c}{ABC} \\
&& \multicolumn{2}{c}{--/12}&\multicolumn{2}{c}{25/12}&\multicolumn{2}{c}{50/12}&\multicolumn{2}{c}{75/12}&\multicolumn{2}{c}{100/12}&\multicolumn{2}{c}{25/12}&\multicolumn{2}{c}{50/12}&\multicolumn{2}{c}{75/12}&\multicolumn{2}{c}{100/12} \\
\midrule
\multirow{7}{*}{\rotatebox[origin=c]{90}{\parbox{2cm}{Electricity\\Market Prices}}}
 & D3 &           $0.11$ &           $0.25$ &           $0.09$ &           $0.16$ &           $0.10$ &           $0.16$ &           $0.12$ &           $0.17$ &           $0.14$ &           $0.20$ &           $0.29$ &           $0.15$ &           $0.22$ &           $0.18$ &           $0.20$ &           $0.17$ &           $0.30$ &           $0.28$ \\
        & HDDDM &           $0.18$ &           $0.38$ &           $0.00$ &           $0.00$ &           $0.01$ &           $0.06$ &           $0.00$ &           $0.04$ &           $0.01$ &           $0.06$ &           $0.18$ &           $0.23$ &           $0.11$ &           $0.20$ &           $0.10$ &           $0.20$ &           $0.10$ &           $0.20$ \\
        & KSDD &           $0.06$ &           $0.09$ &           $0.07$ &           $0.08$ &           $0.08$ &           $0.08$ &           $0.10$ &           $0.09$ &           $0.09$ &           $0.10$ &           $0.18$ &           $0.09$ &           $0.16$ &           $0.10$ &           $0.17$ &           $0.10$ &           $0.19$ &           $0.11$ \\
        & MMDDDM &           $0.09$ &           $0.10$ &           $0.02$ &           $0.06$ &           $0.06$ &           $0.10$ &           $0.05$ &           $0.08$ &           $0.05$ &           $0.08$ &           $0.17$ &           $0.08$ &           $0.19$ &           $0.11$ &           $0.18$ &           $0.10$ &           $0.14$ &           $0.09$ \\
        & SDDM(MT) &  $\mathbf{0.86}$ &  $\mathbf{0.30}$ &  $\mathbf{0.20}$ &  $\mathbf{0.36}$ &  $\mathbf{0.59}$ &  $\mathbf{0.41}$ &  $\mathbf{0.71}$ &  $\mathbf{0.36}$ &  $\mathbf{0.78}$ &  $\mathbf{0.29}$ &  $\mathbf{0.86}$ &  $\mathbf{0.22}$ &  $\mathbf{0.95}$ &  $\mathbf{0.11}$ &  $\mathbf{0.96}$ &  $\mathbf{0.09}$ &  $\mathbf{0.97}$ &  $\mathbf{0.09}$ \\
        & SDDM(RBF) &           $0.39$ &           $0.45$ &           $0.00$ &           $0.00$ &           $0.01$ &           $0.05$ &           $0.03$ &           $0.15$ &           $0.02$ &           $0.11$ &           $0.64$ &           $0.36$ &           $0.62$ &           $0.36$ &           $0.70$ &           $0.35$ &           $0.66$ &           $0.37$ \\
        & ShapeDDM &           $0.59$ &           $0.47$ &           $0.00$ &           $0.04$ &           $0.02$ &           $0.09$ &           $0.02$ &           $0.10$ &           $0.04$ &           $0.14$ &           $0.37$ &           $0.21$ &           $0.29$ &           $0.24$ &           $0.26$ &           $0.25$ &           $0.22$ &           $0.24$ \\
\hline\multirow{7}{*}{\rotatebox[origin=c]{90}{\parbox{2cm}{Forest\\Covertype}}} & D3 &           $0.40$ &           $0.29$ &           $0.21$ &           $0.18$ &           $0.22$ &           $0.17$ &           $0.24$ &           $0.15$ &           $0.56$ &           $0.27$ &           $0.29$ &           $0.15$ &           $0.26$ &           $0.16$ &           $0.28$ &           $0.15$ &           $0.41$ &           $0.30$ \\
        & HDDDM &           $0.16$ &           $0.31$ &           $0.03$ &           $0.11$ &           $0.04$ &           $0.13$ &           $0.05$ &           $0.15$ &           $0.03$ &           $0.12$ &           $0.17$ &           $0.21$ &           $0.14$ &           $0.21$ &           $0.10$ &           $0.19$ &           $0.08$ &           $0.17$ \\
        & KSDD &           $0.05$ &           $0.03$ &           $0.07$ &           $0.03$ &           $0.07$ &           $0.03$ &           $0.09$ &           $0.03$ &           $0.10$ &           $0.02$ &           $0.09$ &           $0.04$ &           $0.08$ &           $0.04$ &           $0.09$ &           $0.04$ &           $0.09$ &           $0.04$ \\
        & MMDDDM &           $0.15$ &           $0.09$ &           $0.11$ &           $0.12$ &           $0.10$ &           $0.11$ &           $0.19$ &           $0.09$ &           $0.19$ &           $0.10$ &           $0.19$ &           $0.13$ &           $0.16$ &           $0.09$ &           $0.20$ &           $0.08$ &           $0.22$ &           $0.10$ \\
        & SDDM(MT) &  $\mathbf{0.90}$ &  $\mathbf{0.19}$ &  $\mathbf{0.69}$ &  $\mathbf{0.32}$ &  $\mathbf{0.89}$ &  $\mathbf{0.13}$ &  $\mathbf{0.86}$ &  $\mathbf{0.14}$ &  $\mathbf{0.86}$ &  $\mathbf{0.15}$ &  $\mathbf{0.66}$ &  $\mathbf{0.24}$ &  $\mathbf{0.85}$ &  $\mathbf{0.22}$ &  $\mathbf{0.89}$ &  $\mathbf{0.17}$ &  $\mathbf{0.89}$ &  $\mathbf{0.18}$ \\
        & SDDM(RBF) &           $0.03$ &           $0.16$ &           $0.00$ &           $0.00$ &           $0.00$ &           $0.00$ &           $0.00$ &           $0.00$ &           $0.00$ &           $0.00$ &           $0.00$ &           $0.00$ &           $0.01$ &           $0.06$ &           $0.02$ &           $0.12$ &           $0.01$ &           $0.07$ \\
        & ShapeDDM &  $\mathbf{0.92}$ &  $\mathbf{0.25}$ &           $0.02$ &           $0.10$ &           $0.15$ &           $0.23$ &           $0.22$ &           $0.25$ &           $0.30$ &           $0.24$ &           $0.47$ &           $0.11$ &           $0.07$ &           $0.18$ &           $0.03$ &           $0.11$ &           $0.02$ &           $0.09$ \\
\hline\multirow{7}{*}{\rotatebox[origin=c]{90}{\parbox{2cm}{Random\\RBF}}} & D3 &           $0.22$ &           $0.30$ &           $0.12$ &           $0.17$ &           $0.13$ &           $0.18$ &           $0.20$ &           $0.17$ &           $0.23$ &           $0.26$ &           $0.15$ &           $0.16$ &           $0.24$ &           $0.18$ &           $0.21$ &           $0.18$ &           $0.32$ &           $0.29$ \\
        & HDDDM &           $0.07$ &           $0.21$ &           $0.01$ &           $0.08$ &           $0.01$ &           $0.07$ &           $0.01$ &           $0.07$ &           $0.04$ &           $0.13$ &           $0.11$ &           $0.17$ &           $0.02$ &           $0.10$ &           $0.03$ &           $0.11$ &           $0.03$ &           $0.12$ \\
        & KSDD &           $0.07$ &           $0.06$ &           $0.08$ &           $0.07$ &           $0.09$ &           $0.07$ &           $0.12$ &           $0.08$ &           $0.12$ &           $0.08$ &           $0.12$ &           $0.08$ &           $0.12$ &           $0.07$ &           $0.11$ &           $0.07$ &           $0.13$ &           $0.07$ \\
        & MMDDDM &           $0.19$ &           $0.11$ &           $0.10$ &           $0.13$ &           $0.15$ &           $0.09$ &           $0.19$ &           $0.11$ &           $0.23$ &           $0.11$ &           $0.23$ &           $0.12$ &           $0.18$ &           $0.08$ &           $0.21$ &           $0.10$ &           $0.20$ &           $0.09$ \\
        & SDDM(MT) &  $\mathbf{0.92}$ &  $\mathbf{0.18}$ &           $0.20$ &           $0.35$ &  $\mathbf{0.74}$ &  $\mathbf{0.32}$ &  $\mathbf{0.86}$ &  $\mathbf{0.21}$ &  $\mathbf{0.89}$ &  $\mathbf{0.19}$ &  $\mathbf{0.56}$ &  $\mathbf{0.22}$ &  $\mathbf{0.80}$ &  $\mathbf{0.27}$ &  $\mathbf{0.86}$ &  $\mathbf{0.22}$ &  $\mathbf{0.88}$ &  $\mathbf{0.20}$ \\
        & SDDM(RBF) &           $0.06$ &           $0.22$ &           $0.00$ &           $0.00$ &           $0.01$ &           $0.08$ &           $0.06$ &           $0.21$ &           $0.09$ &           $0.26$ &           $0.01$ &           $0.07$ &           $0.12$ &           $0.25$ &           $0.15$ &           $0.28$ &           $0.20$ &           $0.34$ \\
        & ShapeDDM &  $\mathbf{0.93}$ &  $\mathbf{0.21}$ &           $0.01$ &           $0.08$ &           $0.08$ &           $0.18$ &           $0.16$ &           $0.23$ &           $0.27$ &           $0.24$ &           $0.48$ &           $0.09$ &           $0.17$ &           $0.23$ &           $0.13$ &           $0.22$ &           $0.09$ &           $0.19$ \\
\pagebreak
\hline\multirow{7}{*}{\rotatebox[origin=c]{90}{\parbox{2cm}{Rotating\\Hyperplane}}} & D3 &           $0.06$ &           $0.19$ &           $0.03$ &           $0.11$ &           $0.01$ &           $0.07$ &           $0.04$ &           $0.12$ &           $0.02$ &           $0.08$ &           $0.08$ &           $0.16$ &           $0.07$ &           $0.15$ &           $0.07$ &           $0.16$ &           $0.10$ &           $0.18$ \\
        & HDDDM &           $0.01$ &           $0.12$ &           $0.01$ &           $0.07$ &           $0.00$ &           $0.00$ &           $0.01$ &           $0.08$ &           $0.01$ &           $0.07$ &           $0.00$ &           $0.04$ &           $0.02$ &           $0.09$ &           $0.02$ &           $0.10$ &           $0.01$ &           $0.06$ \\
        & KSDD &           $0.02$ &           $0.07$ &           $0.06$ &           $0.10$ &           $0.04$ &           $0.09$ &           $0.04$ &           $0.09$ &  $\mathbf{0.05}$ &  $\mathbf{0.09}$ &           $0.04$ &           $0.08$ &           $0.04$ &           $0.08$ &           $0.05$ &           $0.09$ &           $0.03$ &           $0.08$ \\
        & MMDDDM &           $0.17$ &           $0.09$ &           $0.03$ &           $0.09$ &           $0.04$ &           $0.09$ &           $0.04$ &           $0.11$ &           $0.02$ &           $0.06$ &           $0.17$ &           $0.08$ &           $0.19$ &           $0.10$ &           $0.20$ &           $0.08$ &           $0.22$ &           $0.08$ \\
        & SDDM(MT) &           $0.24$ &           $0.41$ &           $0.00$ &           $0.04$ &           $0.00$ &           $0.00$ &           $0.00$ &           $0.03$ &           $0.01$ &           $0.06$ &           $0.18$ &           $0.23$ &           $0.17$ &           $0.24$ &           $0.15$ &           $0.24$ &           $0.18$ &           $0.24$ \\
        & SDDM(RBF) &           $0.75$ &           $0.25$ &           $0.00$ &           $0.00$ &           $0.00$ &           $0.00$ &           $0.00$ &           $0.00$ &           $0.00$ &           $0.00$ &  $\mathbf{0.70}$ &  $\mathbf{0.29}$ &  $\mathbf{0.69}$ &  $\mathbf{0.27}$ &  $\mathbf{0.67}$ &  $\mathbf{0.27}$ &  $\mathbf{0.65}$ &  $\mathbf{0.31}$ \\
        & ShapeDDM &  $\mathbf{0.91}$ &  $\mathbf{0.17}$ &           $0.00$ &           $0.04$ &           $0.01$ &           $0.08$ &           $0.01$ &           $0.06$ &           $0.01$ &           $0.08$ &           $0.43$ &           $0.15$ &           $0.02$ &           $0.11$ &           $0.00$ &           $0.04$ &           $0.00$ &           $0.04$ \\
\hline\multirow{7}{*}{\rotatebox[origin=c]{90}{\parbox{2cm}{STAGGER}}} & D3 &           $0.19$ &           $0.29$ &           $0.17$ &           $0.18$ &           $0.16$ &           $0.19$ &           $0.20$ &           $0.17$ &           $0.27$ &           $0.28$ &           $0.18$ &           $0.17$ &           $0.21$ &           $0.20$ &           $0.24$ &           $0.19$ &           $0.24$ &           $0.25$ \\
        & HDDDM &           $0.14$ &           $0.28$ &           $0.01$ &           $0.08$ &           $0.02$ &           $0.08$ &           $0.07$ &           $0.17$ &           $0.04$ &           $0.13$ &           $0.06$ &           $0.14$ &           $0.09$ &           $0.17$ &           $0.04$ &           $0.13$ &           $0.07$ &           $0.17$ \\
        & KSDD &           $0.22$ &           $0.28$ &           $0.14$ &           $0.17$ &           $0.12$ &           $0.16$ &           $0.18$ &           $0.20$ &           $0.22$ &           $0.24$ &           $0.16$ &           $0.17$ &           $0.17$ &           $0.20$ &           $0.17$ &           $0.19$ &           $0.19$ &           $0.21$ \\
        & MMDDDM &           $0.16$ &           $0.09$ &           $0.16$ &           $0.15$ &           $0.16$ &           $0.11$ &           $0.15$ &           $0.09$ &           $0.15$ &           $0.08$ &           $0.19$ &           $0.10$ &           $0.12$ &           $0.08$ &           $0.18$ &           $0.11$ &           $0.14$ &           $0.07$ \\
        & SDDM(MT) &  $\mathbf{0.87}$ &  $\mathbf{0.32}$ &  $\mathbf{0.46}$ &  $\mathbf{0.47}$ &  $\mathbf{0.64}$ &  $\mathbf{0.43}$ &  $\mathbf{0.72}$ &  $\mathbf{0.42}$ &  $\mathbf{0.65}$ &  $\mathbf{0.46}$ &  $\mathbf{0.68}$ &  $\mathbf{0.33}$ &  $\mathbf{0.69}$ &  $\mathbf{0.38}$ &  $\mathbf{0.79}$ &  $\mathbf{0.36}$ &  $\mathbf{0.82}$ &  $\mathbf{0.33}$ \\
        & SDDM(RBF) &           $0.71$ &           $0.34$ &           $0.11$ &           $0.28$ &           $0.21$ &           $0.29$ &           $0.45$ &           $0.39$ &           $0.46$ &           $0.38$ &           $0.31$ &           $0.33$ &           $0.36$ &           $0.28$ &           $0.48$ &           $0.33$ &           $0.45$ &           $0.29$ \\
        & ShapeDDM &  $\mathbf{0.92}$ &  $\mathbf{0.20}$ &           $0.06$ &           $0.17$ &           $0.09$ &           $0.18$ &           $0.04$ &           $0.14$ &           $0.27$ &           $0.21$ &           $0.24$ &           $0.25$ &           $0.16$ &           $0.23$ &           $0.15$ &           $0.21$ &           $0.22$ &           $0.23$ \\
\hline\multirow{7}{*}{\rotatebox[origin=c]{90}{\parbox{2cm}{Nebraska Weather}}} & D3 &           $0.03$ &           $0.14$ &  $\mathbf{0.12}$ &  $\mathbf{0.17}$ &           $0.14$ &           $0.18$ &           $0.15$ &           $0.17$ &           $0.21$ &           $0.26$ &           $0.23$ &           $0.17$ &           $0.18$ &           $0.17$ &           $0.21$ &           $0.16$ &           $0.30$ &           $0.28$ \\
        & HDDDM &           $0.01$ &           $0.08$ &           $0.00$ &           $0.04$ &           $0.01$ &           $0.07$ &           $0.01$ &           $0.07$ &           $0.01$ &           $0.08$ &           $0.00$ &           $0.00$ &           $0.05$ &           $0.15$ &           $0.03$ &           $0.12$ &           $0.03$ &           $0.12$ \\
        & KSDD &           $0.04$ &           $0.05$ &           $0.07$ &           $0.06$ &           $0.07$ &           $0.06$ &           $0.08$ &           $0.06$ &           $0.09$ &           $0.06$ &           $0.10$ &           $0.04$ &           $0.09$ &           $0.04$ &           $0.10$ &           $0.05$ &           $0.11$ &           $0.04$ \\
        & MMDDDM &           $0.02$ &           $0.07$ &           $0.06$ &           $0.12$ &           $0.06$ &           $0.07$ &           $0.06$ &           $0.08$ &           $0.09$ &           $0.09$ &           $0.18$ &           $0.09$ &           $0.17$ &           $0.10$ &           $0.17$ &           $0.10$ &           $0.19$ &           $0.11$ \\
        & SDDM(MT) &  $\mathbf{0.69}$ &  $\mathbf{0.41}$ &  $\mathbf{0.14}$ &  $\mathbf{0.30}$ &  $\mathbf{0.47}$ &  $\mathbf{0.43}$ &  $\mathbf{0.57}$ &  $\mathbf{0.42}$ &  $\mathbf{0.54}$ &  $\mathbf{0.41}$ &  $\mathbf{0.66}$ &  $\mathbf{0.31}$ &  $\mathbf{0.88}$ &  $\mathbf{0.18}$ &  $\mathbf{0.87}$ &  $\mathbf{0.20}$ &  $\mathbf{0.85}$ &  $\mathbf{0.23}$ \\
        & SDDM(RBF) &           $0.06$ &           $0.22$ &           $0.03$ &           $0.16$ &           $0.10$ &           $0.25$ &           $0.29$ &           $0.37$ &           $0.33$ &           $0.39$ &           $0.49$ &           $0.39$ &           $0.78$ &           $0.23$ &  $\mathbf{0.84}$ &  $\mathbf{0.17}$ &  $\mathbf{0.82}$ &  $\mathbf{0.21}$ \\
        & ShapeDDM &           $0.33$ &           $0.46$ &           $0.01$ &           $0.07$ &           $0.01$ &           $0.08$ &           $0.06$ &           $0.15$ &           $0.07$ &           $0.17$ &           $0.17$ &           $0.23$ &           $0.07$ &           $0.17$ &           $0.06$ &           $0.16$ &           $0.06$ &           $0.16$ \\
\bottomrule
\end{longtable}
\pagebreak
\end{landscape}

\end{toappendix}

As can be seen, SDDM (MT) outperforms all other methods on nearly all datasets. In particular, in the case of multiple concepts (ABA, ABC) SDDM performs significantly better then any other method. An exception is ShapeDDM which performs quite well in the case AB. Furthermore, SDDM (RBF) also shows suitable performance on some datasets. This implies that the spectral method itself produces decent performance but requires a sufficient kernel to work properly. Furthermore, SDDM (MT and RBF) and ShapeDDM produce relatively few alarms (in median 1 or less false alarms per run), in particular, in cases where they perform not that good. This is in strong contrast to other methods like MMDDDM or KS-Win which usually produced in median 10-15 false alarms (and up to 30 in median).  

\section{Conclusion and Future Work}

In this contribution we considered the shape and structure of drift induced signals, obtained by MMD. We gave a complete description, in case the signals are induced by sliding windows, by giving a direct connection to the rate of change in terms of derivatives. We also showed that this problem is closely related to the eigenvectors of the kernel matrix. We used this insight to construct a new drift detection method SDDM which is based on a segmentation of the eigenvectors. We compared SDDM to several other unsupervised drift detection methods on several datasets and showed it superiority. We also introduced the concept of the MomentTree-kernel, which showed high potential in context of several dataset.  

So far, our method is limited to abrupt drift, we hope to extend it to gradual drift in the future. 

\bibliographystyle{abbrv}
\bibliography{bib}


\end{document}